\documentclass{article}
\usepackage[british]{babel}
\usepackage{amsmath}
\usepackage{amsfonts}
\usepackage{amssymb}
\usepackage{mathrsfs}
\usepackage{bbm}
\usepackage{listings}
\usepackage{graphicx}
\usepackage{upgreek}
\usepackage{enumitem}
\usepackage{amsthm}
\usepackage{lineno}
\usepackage{float}
\usepackage{longtable}
\usepackage{psfrag}
\usepackage{color}
\usepackage{import}
\usepackage{wasysym}
\usepackage{acronym}
\usepackage{multirow}
\usepackage{mathtools}
\usepackage{relsize}
\usepackage{array}
\usepackage{microtype}
\usepackage{graphicx}
\usepackage{subfigure}
\usepackage{booktabs} 

\usepackage{hyperref}

\usepackage{breakurl}

\usepackage[accepted]{icml2023}

\usepackage[capitalize,noabbrev]{cleveref}

\theoremstyle{plain}
\newtheorem{theorem}{Theorem}[section]

\newtheorem{lemma}[theorem]{Lemma}
\newtheorem{corollary}[theorem]{Corollary}
\theoremstyle{definition}

\theoremstyle{remark}

\usepackage[textsize=tiny]{todonotes}

\renewcommand{\algorithmicrequire}{\textbf{Input:}}
\renewcommand{\algorithmicensure}{\textbf{Output:}}

\newcommand\V[1]  { \mathbf{#1} }
\newcommand\B[1]  { \boldsymbol{#1} }
\newcommand\up[1] {\mathrm{#1}}
\newcommand\set[1] {\mathcal{#1}}
\newcommand\sset[1] {\mathsmaller{\mathcal{#1}}}

\acrodef{MRC}{minimax risk classifier}
\acrodefplural{MRC}{minimax risk classifiers}
\acrodef{RRM}{robust risk minimization}
\acrodef{KMM}{kernel mean matching}
\acrodef{DW-KMM}{double-weighting \ac{KMM}}
\acrodef{GKMM}{generalized kernel mean matching}
\acrodef{MSE}{mean square error}
\acrodef{RFF}{random fourier features}
\acrodef{RKHS}{reproducing kernel Hilbert space}

\icmltitlerunning{Double-Weighting for Covariate Shift Adaptation}

\begin{document}

\twocolumn[
\icmltitle{Double-Weighting for Covariate Shift Adaptation}



\icmlsetsymbol{equal}{*}

\begin{icmlauthorlist}
\icmlauthor{José I. Segovia-Martín}{BCAM}
\icmlauthor{Santiago Mazuelas}{BCAM,Iker}
\icmlauthor{Anqi Liu}{JHU}
\end{icmlauthorlist}

\icmlaffiliation{JHU}{CS department, Whiting School of Engineering, Johns Hopkins University, Baltimore, Maryland, USA}
\icmlaffiliation{BCAM}{Basque Center for Applied Mathematics (BCAM), Bilbao, Spain}
\icmlaffiliation{Iker}{IKERBASQUE-Basque Foundation for Science}

\icmlcorrespondingauthor{José I. Segovia-Martín}{jsegovia@bcamath.org}
\icmlcorrespondingauthor{Santiago Mazuelas}{smazuelas@bcamath.org}
\icmlcorrespondingauthor{Anqi Liu}{aliu@cs.jhu.edu}

\icmlkeywords{Covariate Shift, Supervised Classification, Selection Bias, Minimax Classification}

\vskip 0.3in
]

\printAffiliationsAndNotice{} 

\begin{abstract}
Supervised learning is often affected by a covariate shift in which the marginal distributions of instances (covariates $x$) of training and testing samples $\up{p}_\text{tr}(x)$ and $\up{p}_\text{te}(x)$ are different but the label conditionals coincide.
Existing approaches address such covariate shift by either using the ratio $\up{p}_\text{te}(x)/\up{p}_\text{tr}(x)$ to weight training samples (reweighted methods) or using the ratio $\up{p}_\text{tr}(x)/\up{p}_\text{te}(x)$ to weight testing samples (robust methods).  
However, the performance of such approaches can be poor under support mismatch or when the above ratios take large values.
We propose a minimax risk classification (MRC) approach for covariate shift adaptation that avoids such limitations by weighting both training and testing samples.
In addition, we develop effective techniques that obtain both sets of weights and  generalize the conventional kernel mean matching method.
We provide novel generalization bounds for our method that show a significant increase in the effective sample size compared with reweighted methods. 
The proposed method also achieves enhanced classification performance in both synthetic and empirical experiments. 
\end{abstract}
\vspace{-0.2cm}
\section{Introduction}
\vspace{-0.1cm}
Most supervised learning methods assume that training and testing samples are drawn i.i.d. from the same underlying distribution.
However, practical scenarios are often affected by a covariate shift in which the marginal distributions of instances (covariates $x$) of training and testing samples 
$\up{p}_\text{tr}(x)$ and $\up{p}_\text{te}(x)$ 
are different (see e.g., \cite{Sugiyama2012,Quinonero2008}), while the conditional label distribution stays the same.
In such scenarios, conventional supervised classification methods, like empirical risk minimization, can perform poorly because the empirical risk is approximating the training expected risk, rather than the test expected risk.

Most of the existing methods for covariate shift adaptation are based on the reweighted approach \cite{Sugiyama2012,Quinonero2008,Cortes2008,Huang2006}. 
These methods weight loss functions at training using the ratio $\up{p}_{\text{te}}(x)/\up{p}_{\text{tr}}(x)$ so that training samples more likely in the test distribution are assigned higher weights (see Fig.~\ref{fig_1:synthetic_balls}), increasing their relevance at training. Such ratios can be estimated by using training and testing instances 
\cite{Tsuboi2009,Yamada2011,Liu2013}.
Reweighted methods are designed for situations where the support of $\up{p}_\text{tr}$ contains that of $\up{p}_\text{te}$. 
However, even if such condition is satisfied, reweighted methods may achieve poor performances if the ratio $\up{p}_{\text{te}}(x)/\up{p}_{\text{tr}}(x)$ take large values at certain training samples, leading to inaccurate estimations of expected losses (see e.g., \cite{Cortes2014,Reddi2015}). 
Such problems can be alleviated by flattening the above ratio \citep{Shimodaira2000,Yamada2011}, by utilizing a regularization term based on the unweighted solution \cite{Reddi2015}, and by directly estimating weights for training samples through \ac{KMM} methods \citep{Gretton2009,Huang2006}.  

Robust methods for covariate shift adaptation \citep{Liu2014,Liu2017,Chen2016} are derived from a distributionally robust learning framework, where the feature expectation matching constraints are obtained from training samples but the adversarial risk is defined on the test distribution. 
Such methods weight feature functions at testing using the ratio $\up{p}_{\text{tr}}(x)/\up{p}_{\text{te}}(x)$ (see Fig.~\ref{fig_1:synthetic_balls}). 
The resulting parametric form produces less confident predictions when testing samples are less likely in the training distribution. 
Robust methods are designed for situations where the support of $\up{p}_\text{te}$ contains that of $\up{p}_\text{tr}$. 
However, even if such condition is satisfied, robust methods may achieve poor performances if the ratio $\up{p}_{\text{tr}}(x)/\up{p}_{\text{te}}(x)$ take large values at certain testing samples, leading to overconfident classification rules.

\begin{figure*}
	\label{fig_1:synthetic_balls}
	\centering
        \psfrag{a}[l][b][0.9]{\hspace{-1.5cm}\textcolor{red}{$\newmoon$} Weights for training samples} 
	\psfrag{b}[l][b][0.9]{\hspace{-.2cm}\textcolor{blue}{$\newmoon$} Weights for testing samples}
        \psfrag{x}[l][b][0.9]{\hspace{-1.7cm}Reweighted approach} 
	\psfrag{v}[l][b][0.9]{\hspace{-1.3cm}Robust approach}
        \psfrag{z}[l][b][0.9]{\hspace{-1.4cm}Proposed approach} 
	\includegraphics[width=0.85\textwidth]{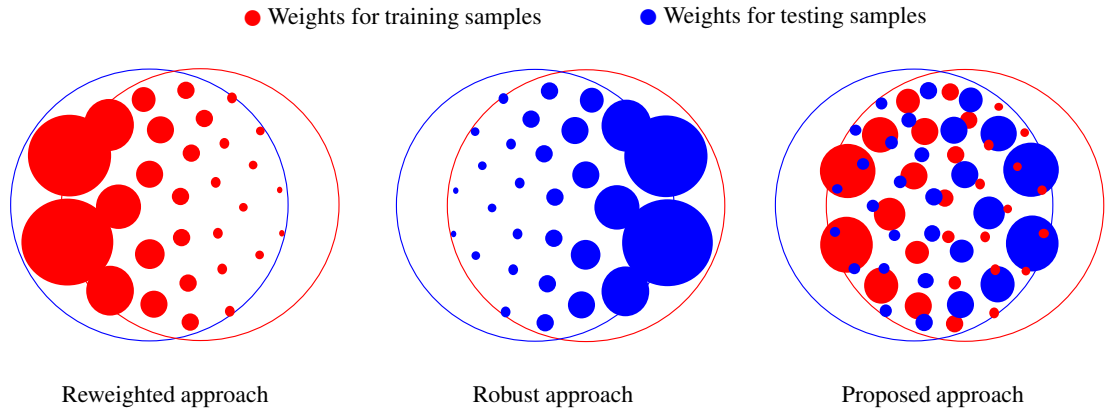}
	\caption{Different approaches for covariate shift adaptation (training and testing instances follow Gaussian distributions with probability mass concentrated in the red and blue circles, resp.). 
	Reweighted methods weight training instances $x$ using the ratio $\beta(x)=\up{p}_{\text{te}}(x)/\up{p}_{\text{tr}}(x)$ while robust methods weight testing instances $x$ using the ratio $\alpha(x)=\up{p}_{\text{tr}}(x)/\up{p}_{\text{te}}(x)$. 
        The proposed approach utilizes weights both for training and testing instances and can avoid large weights $\beta(x)$ by reducing the corresponding $\alpha(x)$ and avoid large weights $\alpha(x)$ by reducing the corresponding $\beta(x)$.}
        \vspace{-0.2cm}
\end{figure*}

In practice, the distributions of training and testing instances can differ in an arbitrary manner  (e.g., their supports may not be contained in each other).
This paper proposes a learning methodology that can tackle such general covariate shift and addresses the limitations of existing approaches by weighting both training and testing samples (see Fig.~\ref{fig_1:synthetic_balls}). 
In particular, the methods proposed are based on \acp{MRC} \citep{Mazuelas2022generalized,Mazuelas2023minimax} and utilize weighted  averages of training samples to estimate expectations of weighted feature functions under the test distribution. 
Specifically, the main contributions in the paper are as follows.

\begin{itemize}	[leftmargin=13pt,topsep=0pt,itemsep=0mm]		
	\item We present a learning framework for general covariate shift adaptation based on a double-weighting of both training and testing samples. Our framework encompasses existing approaches for specific choices of weights.
	
	\item We propose effective techniques that obtain weights for training and testing samples, and generalize the conventional \ac{KMM} that only obtains weights for training samples. 
		
	\item We develop generalization bounds for the proposed methods that show a significant increase in effective sample size compared with reweighted approaches.
 	
	\item We experimentally assess the performance improvement obtained by the proposed techniques in multiple covariate shift scenarios. 
\end{itemize}

\textbf{Notations.} Calligraphic  upper case letters represent sets;
 bold lower and upper case letters represent vectors and matrices, respectively;
 for a vector $\V{v}$, $v^{(i)}$ denotes its $i$-th component, $|\V{v}|$ denotes its component-wise absolute value, and $(\V{v})_{+}$ denotes its positive part;
 $\B{1}$ denotes a vector with all components equal to 1;
 $||\cdot||_1$, $||\cdot||_{\infty}$, and $||\cdot||_{\mathcal{H}}$ denote the 1-norm, the infinity, and the Hilbert space norm of its argument, respectively;
 $\preceq$ and $\succeq$ represent vector component-wise inequalities; $N(\V{m},\B{\Sigma})$ denotes the pdf of a Gaussian r.v. $\V{x}$ with mean $\V{m}$ and covariance matrix $\B{\Sigma}$;  
 and $\mathbb{E}_{\up{p}}\{\cdot\}$ denotes the expectation of its argument w.r.t distribution $\up{p}$.
\vspace{-0.2cm}
\section{Preliminaries}\label{preliminaries}
\vspace{-0.1cm}
This section describes the learning setup, the two main existing approaches for covariate shift adaptation, and the framework of \acp{MRC}.

\textbf{Setup.} Let $\set{X}$ be the set of instances and $\set{Y}$ the set of labels represented by the set $\left\lbrace 1,\ldots,|\set{Y}|\right\rbrace $. 
We denote by $\Delta(\set{X}\times\set{Y})$ the set of probability distributions over $\set{X}$ and $\set{Y}$, and by $\text{T}(\set{X},\set{Y})$ the set of classification rules.
For $\up{h}\in \text{T}(\set{X},\set{Y})$, we denote by $\up{h}(y|x)$ the probability with which instance $x\in\set{X}$ is classified by label $y\in\set{Y}$.
We use the notation $\up{p}_{\text{te}}$ for the underlying distribution at test, and $(x_1,y_1),(x_2,y_2)\ldots,(x_n,y_n)$ for the set of training samples.
The $\ell$-risk of a classification rule $\up{h}$ is its expected classification loss with respect to the true underlying distribution at test $\up{p}_{\text{te}}$, i.e., $ R(\up{h})=\mathbb{E}_{\up{p}_{\text{te}}}\left\lbrace\ell\left(\up{h},(x,y)\right)\right\rbrace$.

The learning objective is to use the training samples to find a classification rule $\up{h}$ that has small $\ell$-risk $R(\up{h})$. In this paper, we consider 0-1-loss and log-loss:
\begin{alignat}{2}\label{eq_2:losses}
    &\ell_{01}\left(\up{h},(x,y)\right)   &&= 1-\up{h}(y|x)\\
    &\ell_{\log}\left(\up{h},(x,y)\right) &&= -\log\up{h}(y|x).
\end{alignat}
\textbf{Covariate shift.}  
Under covariate shift, the training samples $(x_1,y_1),(x_2,y_2),\ldots,(x_n,y_n)$ follow a distribution $\up{p}_{\text{tr}}(x,y)$ such that  the marginal distributions of instances differ, $\up{p}_{\text{te}}(x)\neq \up{p}_{\text{tr}}(x)$, but label conditionals coincide, $\up{p}_{\text{tr}}(y|x)=\up{p}_{\text{te}}(y|x)$.
In addition, covariate shift methods assume that the ratio between $\up{p}_{\text{tr}}(x)$ and $\up{p}_{\text{te}}(x)$ is known or that $t$ unsupervised samples $x_{n+1},x_{n+2},\ldots,x_{n+t}$ from $\up{p}_{\text{te}}(x)$ are known at training. 
In previous literature, it is also usually assumed that the training support contains that at testing or vice versa.   
In this paper, we consider general scenarios of covariate shift in which such supports are not required to contain each other.
\vspace{-0.2cm}
\subsection{Main existing approaches}\label{Section_2.1}
\vspace{-0.1cm}
\textbf{Reweighted methods.} 
Most of the techniques for covariate shift adaptation are based on the reweighted approach \citep{Sugiyama2012,Shimodaira2000,Zadrozny2004,Cortes2008,Dudik2005,Lin2002}.
These methods exploit that, for any function $f$, we have that 
\begin{equation}\label{eq_2:reweighted_expectation}
\mathbb{E}_{\up{p}_{\text{te}}}f(x,y)=\mathbb{E}_{\up{p}_{\text{tr}}}\beta(x)f(x,y),\text{ for }\beta(x)=\frac{\up{p}_{\text{te}}(x)}{\up{p}_{\text{tr}}(x)}
\end{equation}
if $\up{p}_{\text{te}}(x)>0 \Rightarrow \up{p}_{\text{tr}}(x)>0$.
Reweighted methods weight loss functions at training by means of  the weight function $\beta(x)$ in \eqref{eq_2:reweighted_expectation}, as detailed in Appendix~\ref{Appendix_1:Reweighted_Robust}. 
Using these weights, such methods can account for the fact that some training instances are unlikely at testing, and assign low relevance to such instances at training  (see Fig.\ref{fig_1:synthetic_balls}).

Reweighted methods assume the support of $\up{p}_{\text{tr}}$ contains that of $\up{p}_{\text{te}}$ (i.e., $\up{p}_{\text{te}}(x)>0\Rightarrow\up{p}_{\text{tr}}(x)>0$) so that \eqref{eq_2:reweighted_expectation} is valid.
Even if this condition is satisfied, such methods may achieve poor performances if the ratio $\beta(x)$ in \eqref{eq_2:reweighted_expectation} takes large values at certain training samples. 
In these cases, the learning process is dominated by few training samples (see e.g., \cite{Cortes2014,Gretton2009}).
The flattening approach alleviates such problems using weights for training samples smoothed utilizing a hyperparameter $\gamma$ as $\left(\up{p}_{\text{te}}(x)/\up{p}_{\text{tr}}(x)\right)^\gamma$ in \cite{Shimodaira2000} and as $\up{p}_{\text{te}}(x)/\left(\gamma\up{p}_{\text{te}}(x)+(1-\gamma)\up{p}_{\text{tr}}(x)\right)$ in \cite{Yamada2011}.

\textbf{Robust methods.} 
Robust methods under covariate shift \citep{Liu2014,Liu2017} exploit that, for any $f$:
\begin{equation}\label{eq_2:robust_expectation}
	\mathbb{E}_{\up{p}_{\text{te}}}\alpha(x)f(x,y)=\mathbb{E}_{\up{p}_{\text{tr}}}f(x,y)\text{, for }\alpha(x)=\frac{\up{p}_{\text{tr}}(x)}{\up{p}_{\text{te}}(x)}
\end{equation}
if $\up{p}_{\text{tr}}(x)>0 \Rightarrow \up{p}_{\text{te}}(x)>0$.
Robust methods weight feature functions at testing\footnote{The robust bias-aware prediction weight the feature functions in both training and testing as the weight appears in the predictive parametric form. Here we are emphasizing the weights at testing to show a symmetric view between these two methods.} by means of the weight function $\alpha(x)$ in \eqref{eq_2:robust_expectation}, as detailed in Appendix~\ref{Appendix_1:Reweighted_Robust}.
Using these weights, such methods can account for the fact that some testing instances are unlikely at training, and consider rules that assign low-confidence predictions to such instances (see Fig.~\ref{fig_1:synthetic_balls}). 

Robust methods assume the support of $\up{p}_{\text{te}}$ contains that of $\up{p}_{\text{tr}}$ (i.e., $\up{p}_{\text{tr}}(x)>0\Rightarrow\up{p}_{\text{te}}(x)>0$) so that \eqref{eq_2:robust_expectation} is valid. 
Even if this condition is satisfied, such methods may achieve poor performances if the ratio $\alpha(x)$ in \eqref{eq_2:reweighted_expectation} takes large values at certain testing samples. 
In these cases, the classification rule would only provide confident predictions at few testing samples. 

The connection and symmetric relation between reweighted and robust methods are also shown in Theorem~3 in \citep{Liu2014}. 
They can both be regarded as special cases of the adversarial risk minimization framework \citep{Fathony2016} when the feature expectation matching constraints are set to match different empirical estimates of the features.  
To enable covariate shift adaptation in general cases (e.g., training and testing supports not contained in each other), this paper proposes a learning framework that avoids the limitations of existing methods by utilizing a double-weighting approach.
\vspace{-0.2cm}
\subsection{Minimax Risk Classifiers}
\vspace{-0.1cm}
 Similarly to other approaches based on \ac{RRM} \citep{Farnia2016,Fathony2016}, \ac{MRC} methods \citep{Mazuelas2022generalized,Mazuelas2023minimax} do not require that the training samples follow the same distribution as the testing samples.
\acp{MRC} minimize the \mbox{worst-case} expected loss with respect to distributions in uncertainty sets that can contain the true underlying distribution with high probability. 
The uncertainty sets are given by constraints on the expectation of a function $\Phi:\set{X}\times\set{Y}\longrightarrow \mathbb{R}^m$ referred to as feature mapping.
Such a function can be defined using \mbox{one-hot} encodings of the elements of $\set{Y}$ as $\Phi(x,y)=\B{e}_{y}\otimes\V{x}$,
where $\B{e}_y$ is the \mbox{$y$-th} element of the canonical basis of $\mathbb{R}^{|\set{Y}|}$ and $\otimes$ denotes the Kronecker product.

Given the uncertainty set $\set{U}$, we say that a classification rule $\up{h}^{\sset{U}}$ is a $\ell$-\ac{MRC} for $\set{U}$ if 
\begin{equation}\label{eq_2:classification_rule}
   \up{h}^{\sset{U}} \in\arg\min_{\up{h}\in \text{T}(\set{X},\set{Y})}\max_{\up{p}\in\set{U}}\ell(\up{h},\up{p})
\end{equation}
and, we denote by $R(\set{U})$ the minimax risk against $\set{U}$, i.e.,
\begin{equation}\label{eq_2:Risk_l(U)}
    R(\set{U})=\min_{\up{h}\in \text{T}(\set{X},\set{Y})}\max_{\up{p}\in\set{U}}\ell(\up{h},\up{p})
\end{equation}
where $\ell(\up{h},\up{p})$ denotes the expected loss of classification rule $\up{h}$ w.r.t. distribution $\up{p}$.
\vspace{-0.2cm}
\section{Framework for Adaptation to General Covariate Shift}\label{Section_3}
\vspace{-0.1cm}
This section first describes the proposed double-weighting approach and the corresponding \ac{MRC} learning methodology.  We then describe its relationship with existing techniques,  present finite-sample generalization bounds for the proposed methods, and discuss the trade-off involved in the choice of weight functions.
\vspace{-0.2cm}
\subsection{Double-weighting}\label{Section_3.1}
\vspace{-0.1cm}
The proposed framework considers both training and testing weights, $\beta(x)$ and $\alpha(x)$ (see Fig.\ref{fig_1:synthetic_balls}). 
We exploit the fact that, for any function $f$, we have that 
\begin{equation}\label{eq_3:dw_expectation}
	\mathbb{E}_{\up{p}_{\text{te}}}\alpha(x)f(x,y)=\mathbb{E}_{\up{p}_{\text{tr}}}\beta(x)f(x,y)
\end{equation}
can be attained by multiple choices of weights $\alpha(x)$ and $\beta(x)$. 
For instance, it is satisfied taking
\begin{equation} \label{eq_3:alphabeta_sol}
	\begin{split}
		\alpha(x)=\min\left( C\frac{\up{p}_{\text{tr}}(x)}{\up{p}_{\text{te}}(x)},1\right),
		\beta(x)=\min\left( \frac{\up{p}_{\text{te}}(x)}{\up{p}_{\text{tr}}(x)},C\right)
	\end{split}
\end{equation}
for any $C>0$, since $\alpha(x)\up{p}_{\text{te}}(x)=\beta(x)\up{p}_{\text{tr}}(x)$, $\forall x\in\set{X}$.
Notice that the equality in \eqref{eq_3:dw_expectation} is satisfied taking weights as in \eqref{eq_3:alphabeta_sol} even if the supports of $\up{p}_{\text{tr}}$ and $\up{p}_{\text{te}}$ do not contain each other.

Such a double-weighting approach can avoid the limitations of reweighed and robust methods. For $x\in\set{X}$ with large ratio $\up{p}_{\text{te}}(x)/\up{p}_{\text{tr}}(x)$, using a small $\alpha(x)$ can enable to have $\alpha(x)\up{p}_{\text{te}}(x)=\beta(x)\up{p}_{\text{tr}}(x)$ with moderate values of $\beta(x)$.
Reciprocally, for $x\in\set{X}$ with large ratio $\up{p}_{\text{tr}}(x)/\up{p}_{\text{te}}(x)$, using a small $\beta(x)$ can enable to have \mbox{$\alpha(x)\up{p}_{\text{te}}(x)=\up{p}_{\text{tr}}(x)\beta(x)$} with moderate values of $\alpha(x)$.
For instance, using weights as in \eqref{eq_3:alphabeta_sol} we have that \mbox{$\beta(x)\leq C$} and $\alpha(x)\leq 1$ for any $x\in\set{X}$.
Considering both weights $\beta(x)$ and $\alpha(x)$, we can \emph{both} assign low relevance to training instances that are unlikely at testing, \emph{and also} assign low-confidence predictions to testing instances that are unlikely at training.
\vspace{-0.2cm}
\subsection{MRC learning framework using double-weighting}\label{framework}
\vspace{-0.1cm}
The proposed framework adapts to general covariate shift by constructing the uncertainty set $\set{U}$ in \eqref{eq_2:classification_rule} using both weights $\alpha(x)$ and $\beta(x)$. 
In particular, we use  feature mappings weighted by $\alpha(x)$ as $\Phi_\alpha(x,y)=\alpha(x)\Phi(x,y)$ and constrain the difference between the expectation and empirical mean estimates of feature mappings as follows
\begin{alignat}{1}\label{uncertainty_set}
		\set{U}=\left\lbrace\right.&\up{p}\in\Delta\left(\set{X}\times\set{Y}\right):\left|\mathbb{E}_{\up{p}}\Phi_\alpha(x,y)-\B{\tau}\right|\preceq\B{\lambda}\notag\\
		&\left.\text{and }\up{p}(x)=\up{p}_{\text{te}}(x),\forall x\in\mathcal{X}\right\rbrace
\end{alignat}
where $\B{\tau}$ denotes the mean vector of expectation estimates, and $\B{\lambda}$ is a vector that determines the confidence with which $\up{p}_\text{te}(x,y)\in\set{U}$. 
The expectation of the feature mapping $\Phi_\alpha(x,y)$ is estimated using averages of training samples weighted by $\beta(x)$ as
\begin{align}\label{eq_3:tau}
\hspace{-0.1cm}\B{\tau}=\frac{1}{n}\sum_{i=1}^n\Phi_\beta(x_i,y_i),\, \mbox{for }\Phi_\beta(x,y)=\beta(x)\Phi(x,y).
\end{align}
Notice that the mean vector $\B{\tau}$ is an unbiased estimator of $\mathbb{E}_{\up{p}_{\text{te}}} \Phi_{\alpha}(x,y)$ for any choice of weights satisfying $\alpha(x)\up{p}_{\text{te}}(x)=\beta(x)\up{p}_{\text{tr}}(x)$, in particular for those given by \eqref{eq_3:alphabeta_sol}. 
In addition, the accuracy of $\B{\tau}$ can be improved using weights $\alpha(x)$ that avoid large weights $\beta(x)$, as discussed in  Section~\ref{sec_3.4} below.

\textbf{Convex optimization.} We next show how \acp{MRC} corresponding with uncertainty sets \eqref{uncertainty_set} can be learned by solving the convex optimization problem
\begin{equation}\label{eq_3:CSMRC}
	\min_{\B{\mu}}-\B{\tau}^\text{T}\B{\mu}+\mathbb{E}_{\up{p}_{\text{te}}(x)}\varphi_{\ell}(\B{\mu},x,\alpha(x))+\B{\lambda}^\text{T}|\B{\mu}|
\end{equation}
where $\varphi_{\ell}$ is a function defined under different loss functions. 
For 0-1-loss, we have
\begin{equation}
	\label{varphi_01}
	\varphi_{01}(\B{\mu},x,\alpha(x))=1+\max_{\set{C}\subseteq\set{Y}}\frac{\sum_{y\in\set{C}}\Phi_\alpha(x,y)^\text{T}\B{\mu}-1}{|\set{C}|}
\end{equation}
and, for log-loss, we have
\begin{equation}
	\label{varphi_log}
	\varphi_{\log}(\B{\mu},x,\alpha(x))=\log\sum_{y\in\set{Y}}\exp\left\lbrace\Phi_{\alpha}(x,y)^\text{T}\B{\mu}\right\rbrace.
\end{equation}

\begin{theorem}\label{th_3:MRC}
    Let $\B{\tau},\B{\lambda}\in\mathbb{R}^m$ be such that the uncertainty set $\set{U}$ in \eqref{uncertainty_set} is not the empty set.
    If $\B{\mu}^*$ is a solution of \eqref{eq_3:CSMRC} for 0-1-loss, the classification rule 
    \begin{equation}\label{eq_3:h_equality_01}
        \up{h}^{\sset{U}}(y|x)=\left(\alpha(x)\Phi(x,y)^\text{T}\B{\mu}^*-\varphi_{01}(\B{\mu}^*,x,\alpha(x))+1\right)_{+} 
    \end{equation}
    is a 0-1-\ac{MRC} for $\set{U}$.
    If $\B{\mu}^*$ is a solution of \eqref{eq_3:CSMRC} for log-loss, the classification rule
    \begin{equation}\label{eq_3:h_equality_log}
        \up{h}^{\sset{U}}(y|x)=\exp\left\lbrace\alpha(x)\Phi(x,y)^\text{T}\B{\mu}^*-\varphi_{\log}(\B{\mu}^*,x,\alpha(x))\right\rbrace
    \end{equation}
    is a log-\ac{MRC} for $\set{U}$.
    In addition, the minimax risk $R(\set{U})$ is given by
    \begin{equation}\label{eq_3:R_l(U)}
        R(\set{U})=-\B{\tau}^\text{T}\B{\mu}^*+\mathbb{E}_{\up{p}_{\text{te}}(x)}\varphi_{\ell}(\B{\mu}^*,x,\alpha(x))+\B{\lambda}^\text{T}|\B{\mu}^*|.
    \end{equation}
\end{theorem}
\vspace{-0.4cm}
\begin{proof}
    See Appendix~\ref{Appendix_2:Proofs_Sect3}.
\end{proof}
\vspace{-0.3cm}

\textbf{Remarks.} 
The optimization in \eqref{eq_3:CSMRC} can be addressed in practice using conventional optimization methods such as stochastic gradient descent. 
If unlabeled instances from the test distribution are available at training, they can directly be used to obtain samples corresponding to the (sub)gradient of $\mathbb{E}_{\up{p}_{\text{te}}(x)}\varphi_{\ell}(\B{\mu},x,\alpha(x))$ since the function $\varphi_{\ell}$ does not depend on labels. 
If the marginals $\up{p}_\text{tr}(x)$, $\up{p}_\text{te}(x)$ are known, training samples can be used to obtain samples of the above gradient using \eqref{eq_2:reweighted_expectation}. This theorem is novel as we apply weights $\alpha$ and $\beta$ for covariate shift adaptation, even though the general form is analogous to the results in \cite{Mazuelas2022generalized,Mazuelas2023minimax}, which studies \ac{MRC} with train and test data sampled i.i.d. from the same distribution.

\textbf{Regularization.} 
The convex optimization problem \eqref{eq_3:CSMRC} carries out an L1-type regularization, where the regularization parameter is given by vector $\B{\lambda}$. 
The regularization term in \eqref{eq_3:CSMRC} allows to penalize each component of parameter $\B{\mu}$ differently, such that feature components with poorly estimated expectations (i.e.,  components $i$ with large $\lambda^{(i)}$) are strongly penalized.

\textbf{Classification rule.} 
The deterministic classifier associated with $\up{h}^{\sset{U}}$ classifies each instance with the label maximizing $\up{h}^{\sset{U}}(y|x)$. 
For both losses, this deterministic classifier is given by
 \vspace{-0.15cm}\begin{equation}
\begin{split}
        \arg\max_{y\in\set{Y}}\up{h}^{\sset{U}}(y|x)&=\arg\max_{y\in\set{\set{Y}}}\alpha(x)\Phi(x,y)^\text{T}\B{\mu}^*\\
        &=\arg\max_{y\in\set{\set{Y}}}\Phi(x,y)^\text{T}\B{\mu}^*.
    \end{split}
 \vspace{-0.1cm}
\end{equation}
Such deterministic classifiers, denoted by $\up{h}^{\sset{U}}_\text{d}$, allow us to classify testing samples even if we do not know the weights $\alpha(x)$ associated with them.

\textbf{Predictive confidence.} 
The values of $\alpha(x)$ adjust the confidence with which each sample is classified. 
For instance, for very small values of $\alpha(x)$, the classifier $\up{h}^{\sset{U}}$ uniformly assigns labels in the set $\set{Y}$ for both losses, i.e., \mbox{$\up{h}^{\sset{U}}(y|x)=1/|\set{Y}|$} for all $y\in\set{Y}$.

\textbf{Relation with existing approaches.} 
The general framework proposed above encompasses existing approaches, as detailed in Appendix~\ref{Appendix_1:Reweighted_Robust} for binary classification with log-loss. 
The usage of weights \mbox{$\alpha(x)=1$, $\beta(x)=\up{p}_{\text{te}}(x)/\up{p}_{\text{tr}}(x)$} leads to reweighted methods \cite{Sugiyama2012}, approximating $\mathbb{E}_{\up{p}_{\text{te}}(x)}\varphi_{\ell}(\B{\mu},x,\alpha(x))$ in \eqref{eq_3:CSMRC} using training instances.
The usage of weights \mbox{$\alpha(x)=\up{p}_{\text{tr}}(x)/\up{p}_{\text{te}}(x)$, $\beta(x)=1$} leads to robust methods \citep{Liu2014}, approximating the gradient of $\mathbb{E}_{\up{p}_{\text{te}}(x)}\varphi_{\ell}(\B{\mu},x,\alpha(x))$ in \eqref{eq_3:CSMRC} using training instances.
\vspace{-0.3cm}
\subsection{Generalization bounds}
\vspace{-0.2cm}
The following shows the generalization bounds of the proposed methods in Section~\ref{framework}.
Such bounds are given in terms of smallest minimax risk, $R^\infty$, that corresponds with the uncertainty set given by the exact expectations, and is defined by 
\begin{equation}\label{MRC_infty}
    R^{\infty}=\min_{\B{\mu}}-\mathbb{E}_{\up{p}_{\text{te}}}\Phi_{\alpha}(x,y)^\text{T}\B{\mu}+\mathbb{E}_{\up{p}_{\text{te}}(x)}\varphi_{\ell}(\B{\mu},x,\alpha(x)).
\end{equation}
The \ac{MRC} corresponding to that smallest minimax risk $R^\infty$ could only be obtained by an exact estimation of the expectation of the feature mapping $\Phi_{\alpha}$ that in turn would require an infinite amount of training samples.
The theorem below shows risk bounds for the proposed \acp{MRC} in terms of minimax risks $R(\set{U})$ and smallest minimax risks $R^{\infty}$.

\begin{theorem} \label{th_3:generalization_bound}	
    Let $\set{U}$ be a non-empty uncertainty set given by \eqref{uncertainty_set} and $\up{h}^{\sset{U}}$ be an $\ell$-\ac{MRC} for  $\set{U}$. 
    If $\B{\mu}^*$ and $\B{\mu}_\infty$ are solutions to \eqref{eq_3:CSMRC} and \eqref{MRC_infty}, respectively, then, we have that
    \begin{alignat}{2}
        R\left(\up{h}^{\sset{U}}\right)&\leq &&R(\set{U})+\left(\left|\B{\tau}-\mathbb{E}_{\up{p}_{\text{te}}}\Phi_{\alpha}(x,y)\right|-\B{\lambda}\right)^T\left|\B{\mu}^*\right|\label{eq_3:th_3_bound1}\\
	R\left(\up{h}^{\sset{U}}\right)&\leq &&R^{\infty}+\B{\lambda}^T(|\B{\mu}_\infty|-|\B{\mu}^*|)\notag\\
        & &&+\left|\B{\tau}-\mathbb{E}_{\up{p}_{\text{te}}}\Phi_{\alpha}(x,y)\right|^T\left|\B{\mu}_{\infty}-              
        \B{\mu}^*\right|.\label{eq_3:th_3_bound2}
    \end{alignat}
\end{theorem}
\begin{proof}
    See Appendix~\ref{Appendix_2:Proofs_Sect3}.
\end{proof}
\vspace{-0.2cm}
Note that the minimax risk $R(\set{U})$ obtained at learning offers an upper bound for the $\ell$-risk if \mbox{$\B{\lambda}\succeq\left|\B{\tau}-\mathbb{E}_{\up{p}_{\text{te}}}\Phi_{\alpha}(x,y)\right|$} and an approximate upper bound for general $\B{\lambda}$. 
In addition, the difference between the risk $R\left(\up{h}^{\sset{U}}\right)$ and the smallest minimax risk $R^{\infty}$ decreases with the estimation error \mbox{$|\B{\tau}-\mathbb{E}_{\up{p}_{\text{te}}}\Phi_{\alpha}(x,y)|$}.

We next show how the proposed methods can lead to a significant increase in effective size compared with reweighted methods. 

\begin{corollary}\label{cor3.3}
    Let $\set{U}$ be a non-empty uncertainty set given by \eqref{uncertainty_set} with $\B{\lambda}=0$, and $\up{h}^{\sset{U}}$ be an $\ell$-\ac{MRC} for $\set{U}$. 
    If weights $\alpha(x)$ and $\beta(x)$ are given by \eqref{eq_3:alphabeta_sol} with $C=B/\sqrt{D}$ for $D\geq 1$ and
    \vspace{-0.1cm}
    \begin{align}\label{B_constant}
        B=\sup_{x\in\set{X}}\up{p}_{\text{te}}(x)/\up{p}_{\text{tr}}(x).\vspace{-0.1cm}
    \end{align}
    Then, with probability at least $1-\delta$ we have that 
    \begin{align}
        R(\up{h}^{\sset{U}})\leq R^{\infty}+M\|\B{\mu}_{\infty}-\B{\mu}^*\|_\infty \sqrt{2\frac{B^2}{Dn}\log\frac{2}{\delta}}
    \end{align}
    where $M$ is a constant satisfying $\|\Phi(x,y)\|_\infty\leq M$ for all $x\in\set{X}$, $y\in\set{Y}$.
\end{corollary}
\vspace{-0.4cm}
\begin{proof}
    A direct consequence of Theorem~\ref{th_3:generalization_bound} and Hoeffding's inequality.
\end{proof}
\vspace{-0.2cm}
As described in \cite{Cortes2010,Yu2012}, reweighted methods have an estimation error of the order $\sqrt{2\frac{B^2}{n}\log\frac{2}{\delta}}$ so that the methods proposed can achieve an effective sample size increased by a factor of $D$ using the double-weighting given by \eqref{eq_3:alphabeta_sol} with \mbox{$C=B/\sqrt{D}$}. 
The next section more broadly discusses such an increase in effective sample size and the corresponding trade-off for predictions' confidence.
\vspace{-0.2cm}
\subsection{Choice of weight functions}\label{sec_3.4}
\vspace{-0.1cm}
Existing reweighted and robust methods, as well as the proposed general framework, utilize weights $\alpha(x)$ and $\beta(x)$ in the estimation of expectations:
\vspace{-0.1cm}
\begin{equation}\label{eq_3.2:expectation_estimate}
    \frac{1}{n}\sum_{i=1}^n\beta(x_i)f(x_i,y_i)\approx\mathbb{E}_{\up{p}_{\text{te}}}\alpha(x)f(x,y).
    \vspace{-0.1cm}
\end{equation}
The error of such estimates is determined by the weights $\beta(x)$. 
If $\alpha(x)$ and $\beta(x)$ satisfy \mbox{$\alpha(x)\up{p}_{\text{te}}(x)=\beta(x)\up{p}_{\text{tr}}(x)$}, using Hoeffding's inequality we have that
\vspace{-0.1cm}
\begin{alignat}{1}\label{eq_3.2:bound}
    \biggl|\frac{1}{n}\sum_{i=1}^n\beta(x_i)f(x_i,y_i)&-\mathbb{E}_{\up{p}_{\text{te}}}\alpha(x)f(x,y)\biggr|\notag\\
    &\leq||f||_\infty\sqrt{2\frac{||\beta||^2_\infty}{n}\log\frac{2}{\delta}}\vspace{-0.2cm}
\end{alignat} 
with probability at least $1-\delta$.

In particular, for reweighted methods the bound \eqref{eq_3.2:bound} becomes $||f||_\infty\sqrt{2\frac{B^2}{n}\log\frac{2}{\delta}}$ with $B$ given by \eqref{B_constant} as shown in \cite{Cortes2010,Yu2012}. 

The error in the expectations estimates in \eqref{eq_3.2:expectation_estimate} decreases when we choose the weights $\alpha(x)$ adequately. 
In particular, using small values of $\alpha(x)$ we can achieve \mbox{$\alpha(x)\up{p}_{\text{te}}(x)=\beta(x)\up{p}_{\text{tr}}(x)$} with moderate values of $\beta(x)$.
Such improvement comes at the expense of using classification rules with significant confidence only in the subregion of $\set{X}$ in which $\alpha(x)$ is significantly larger than $0$.

The above trade-off between error in expectations estimates and confidence of classification rules can be addressed using pairs of weights of the form \eqref{eq_3:alphabeta_sol} and varying the value of $C$.
For values $C\geq B$, $\alpha(x)=1$ and $\beta(x)=\up{p}_{\text{te}}(x)/\up{p}_{\text{tr}}(x)$ that corresponds to the reweighted approach. 
For values $C<B$, the expectations' estimates improve as we decrease $C$ since $\|\beta\|_\infty=C$. 
However, the corresponding classification rules would only predict with significant confidence in the subregion of $\set{X}$ where $\alpha(x)$ is significantly larger than $0$. Such subregion shrinks when $C$ decreases because it is composed by the $x\in\set{X}$ where $\up{p}_{\text{te}}(x)$ is not significantly larger than $C\up{p}_{\text{tr}}(x)$.
In the following, we present methods that obtain weights $\alpha$ and $\beta$ addressing the above trade-off, and generalize conventional \ac{KMM} methods.
\vspace{-0.2cm}
\section{Double-weighting Kernel Mean Matching} \label{Section_4}
\vspace{-0.1cm}
The \ac{KMM} method obtains weights $\B{\beta}\in\mathbb{R}^n$ for $n$ training instances $x_1,x_2,\ldots,x_n$ using $t$ testing instances $x_{n+1},x_{n+2},\ldots,x_{n+t}$ \citep{Huang2006,Gretton2009}. 
We propose the \ac{DW-KMM} method that  obtains weights $\B{\beta}\in\mathbb{R}^n$ for the $n$ training instances together with weights $\B{\alpha}\in\mathbb{R}^t$ for the $t$ testing instances by solving the optimization problem
\vspace{-0.2cm}\begin{alignat}{1}\label{eq_4:DW-KMMempirical}
    \min_{\substack{\B{\alpha},\B{\beta}}} \  \  & \Bigg\| \frac{1}{t}\sum_{i=1}^{t}\alpha^{(i)}K(x_{n+i})-\frac{1}{n}\sum_{i=1}^n\beta^{(i)}K(x_i) \Bigg\| ^2_\mathcal{H}\nonumber\\\mbox{\hspace{-0.2cm}s.t. }\  &0\leq\beta^{(i)}\leq B/\sqrt{D},\text{ for }i=1,\ldots,n\notag\\
    &0\leq \alpha^{(i)}\leq1,\text{ for }i=1,\ldots,t\notag\\   
    &\bigg|\frac{1}{n}\sum_{i=1}^n\beta^{(i)}-\frac{1}{t}\sum_{i=1}^{t}\alpha^{(i)}\bigg|\leq \epsilon\notag\\
    &\left| \left|\B{\alpha}-\mathbf{1} \right| \right|\leq \left(1-\frac{1}{\sqrt{D}} \right) \sqrt{t}
\end{alignat}
where $K:\set{X}\longrightarrow\set{H}$ is a feature map corresponding with a \ac{RKHS} $\set{H}$ with kernel $k(x,\bar{x})=\langle K(x),K(\bar{x})\rangle_{\set{H}}$.

As described above, the hyperparameter $D\geq1$ in \eqref{eq_4:DW-KMMempirical} balances the trade-off between error in expectation estimates and confidence of the classification.
For $D=1$, the optimization problem becomes that of \ac{KMM} for reweighted methods \citep{Huang2006,Gretton2009}.

\textbf{Performance guarantees.} 
The proposed approach is an empirical version of the following (population) problem given by exact expectations
\begin{align}\label{eq_4:GKMM_pop}
\vspace{-0.3cm}
        \min_{\alpha(x),\beta(x)}\  &\bigg\|\mathbb{E}_{\up{p}_{\text{te}}(x)}\alpha(x)K(x)-\mathbb{E}_{\up{p}_{\text{tr}}(x)}\beta(x)K(x)\bigg\|^2_{\set{H}}\nonumber\\
        \text{s.t}\  \  \  \  \  \  &0\leq\beta(x)\leq B/\sqrt{D},\  0\leq\alpha(x)\leq1,\  \forall x\in\set{X}\nonumber\\
        &\mathbb{E}_{\up{p}_{\text{te}}(x)}\alpha(x)=\mathbb{E}_{\up{p}_{\text{tr}}(x)}\beta(x)\nonumber\\
        &\mathbb{E}_{\up{p}_{\text{te}}(x)}\left\lbrace(\alpha(x)-1)^2\right\rbrace\leq \left(1-1/\sqrt{D} \right)^2.
        \vspace{-0.1cm}
\end{align}
The minimum value of \eqref{eq_4:GKMM_pop} is zero since \eqref{eq_3:alphabeta_sol} with \mbox{$C=B/\sqrt{D}$} is a feasible solution. 
Then, solutions of \eqref{eq_4:GKMM_pop}, $\hat{\beta}(x)$, $\hat{\alpha}(x)$, provide consistent estimators of expectations because 
\begin{equation}\label{eq_4:consistent}
     \mathbb{E}_{\up{p}_{\text{te}}(x,y)}\hat{\alpha}(x)\Phi(x,y)=\mathbb{E}_{\up{p}_{\text{tr}}(x,y)}\hat{\beta}(x)\Phi(x,y)
\end{equation}
is satisfied if the kernel $k$ is characteristic or if $\mathbb{E}_{\up{p}_{\text{te}}(y|x)}\Phi(x,y)$ belongs to $\set{H}$, analogously as shown in \cite{Yu2012}.

With finite samples, the following theorem shows bounds for the difference between the empirical means in feature space for solutions of \eqref{eq_4:GKMM_pop}.

\begin{theorem}\label{th_4:bound_GKMM}
    If $\hat{\beta}(x)$ and $\hat{\alpha}(x)$ are solutions of \eqref{eq_4:GKMM_pop},  with probability at least $1-\delta$ we have that
    \begin{alignat}{1}
        \Bigg| \Bigg|\frac{1}{n}\sum_{i=1}^n\hat{\beta}(x_i&)K(x_i)-\frac{1}{t}\sum_{i=1}^{t}\hat{\alpha}(x_{n+i})K(x_{n+i}) \Bigg| \Bigg|_\mathcal{H}\notag\\
        &\leq \left(1+\sqrt{2\log\frac{2}{\delta}}\right)  \kappa\sqrt{\left(\frac{B^2}{Dn}+\frac{1}{t} \right) }
    \end{alignat}
    where the constant $\kappa$ satisfies $|k(x,x)|\leq \kappa^2$ for all $x\in\set{X}$.
\end{theorem}
\vspace{-0.2cm}
\begin{proof}
    See Appendix \ref{Appendix_3:Proofs_Sect4}.
\end{proof}
\vspace{-0.2cm}
\textbf{Relation with conventional \ac{KMM}.} 
The solutions $\hat{\beta}(x)$ for conventional \ac{KMM} in reweighted methods satisfy with probability at least $1-\delta$
    \begin{alignat}{1}
        \Bigg| \Bigg|\frac{1}{n}\sum_{i=1}^n\hat{\beta}(x_i)&K(x_i)-\frac{1}{t}\sum_{i=1}^{t}K(x_{n+i}) \Bigg| \Bigg|_\mathcal{H}\notag\\
        &\leq \left(1+\sqrt{2\log\frac{2}{\delta}}\right)  \kappa\sqrt{\left(\frac{B^2}{n}+\frac{1}{t} \right) }
    \end{alignat}
 as shown in Lemma 4 of \citep{Huang2006} and equation (10) in \citep{Yu2012}. 
 Therefore, the proposed \ac{DW-KMM} allows to significantly improve the effective sample by exploiting the usage of weights $\alpha$. 
 Analogously to the results shown in Section~\ref{sec_3.4}, the effective sample size of the methods proposed is $D$ times larger than that of existing \ac{KMM} for reweighted methods.
\vspace{-0.2cm}
\section{Practical Algorithm}\label{Section_5}
\vspace{-0.1cm}
In this section, we present a practical algorithm for the proposed Double-Weighting for General Covariate Shift \mbox{(DW-GCS)}, detailed in Algorithm~\ref{alg:methodology}.  
We first compute weights $\B{\alpha}$ and $\B{\beta}$ by solving \eqref{eq_4:DW-KMMempirical}, then, we learn the classifier's parameters by solving \eqref{eq_3:CSMRC} using mean vector $\B{\tau}$ defined in \eqref{eq_3:tau} and confidence vector $\B{\lambda}$.

\begin{algorithm}
	\caption{The proposed algorithm: DW-GCS}
    \label{alg:methodology}
	\begin{algorithmic}[1]
        \REQUIRE        
		\hspace{0.25cm}Training samples $(x_1,y_1),(x_2,y_2),\ldots,(x_n,y_n)$
		
        \hspace{0.8cm}Testing instances $x_{n+1},x_{n+2},\ldots,x_{n+t}$, $D$
        \ENSURE
        Weights  $\hat{\B{\beta}}$ and  $\hat{\B{\alpha}}$\\
          \hspace{0.8cm}Classifier parameters  $\B{\mu}^*$, Minimax risk $R(\set{U})$	
        \STATE $\hat{\B{\beta}},\hat{\B{\alpha}} \gets \text{solution of \eqref{eq_4:DW-KMMempirical}}$
		\STATE $\B{\tau} \gets \frac{1}{n}\sum_{i=1}^n\hat{\beta}^{(i)}\Phi(x_i,y_i)$	
		\STATE $\B{\lambda} \gets \text{solution of \eqref{eq_5:optprob_lambda}}$
		\STATE $\B{\mu}^* \gets \text{solution of \eqref{eq_5:CSMRC_empirical}}$ using \eqref{varphi_01} for 0-1-loss, and \eqref{varphi_log} for log-loss
		\STATE $R(\set{U})\gets -\B{\tau}^\text{T}\B{\mu}^*+\frac{1}{t}\overset{t}{\underset{i=1}{\sum}}\varphi_{\ell}(\B{\mu}^*,x_{n+i},\hat{\alpha}^{(i)})+\B{\lambda}^\text{T}|\B{\mu}^*|$	
	\end{algorithmic} 
\end{algorithm}

\textbf{Computing weights and learning \acp{MRC}.}
Weights $\B{\alpha}$ and $\B{\beta}$ are computed solving the convex optimization \eqref{eq_4:DW-KMMempirical}, which is a quadratic problem as detailed in Appendix~\ref{Appendix_4:Quadratic_DKMM}.

The optimization in \eqref{eq_3:CSMRC} can be addressed by approximating the expectation by means of the $t$ instances in testing $x_{n+1},x_{n+2},\ldots,x_{n+t}$ as
\begin{equation}
	\label{eq_5:CSMRC_empirical}
	\min_{\B{\mu}}-\B{\tau}^T\B{\mu}+\frac{1}{t}\sum_{i=1}^{t}\varphi_{\ell}(\B{\mu},x_{n+i},\alpha^{(i)})+\B{\lambda}^T|\B{\mu}|
\end{equation}
that is an unconstrained convex optimization problem and can be efficiently solved by conventional methods.

\textbf{Hyperparameters.}
In principle, both hyperparameters $\B{\lambda}$ and $D$ can be obtained by cross-validation. 
However, standard cross-validation is not valid  under covariate shift \citep{sugiyama2007}. 
We hence avoid cross-validation and determine both parameters as follows.

As detailed in Section~\ref{sec_3.4}, the hyperparameter $D$ serves to address the trade-off between error in expectation estimates and confidence of classification rules. 
For instance, values of $D$ close to 1 can be effective in situations with a large number of samples while higher values of $D$ can be effective with a reduced number of samples. 
This is shown by the theoretical results in the paper, since the estimation error in the proposed methods is of the order $\set{O}(1/\sqrt{Dn})$, as described by the performance bounds in Corollary~\ref{cor3.3} and Theorem~\ref{th_4:bound_GKMM}. 

In practice, we propose to select the hyperparameter $D$ taking advantage of the minimax risk provided at the learning stage by the methods presented. 
Specifically, we select the value of $D$ to achieve the lowest minimax risk over a certain range $D\geq 1$. 
Note that, as described in Theorem~\ref{th_3:generalization_bound}, the minimax risk $R(\set{U})$ obtained at learning offers an upper bound for the risk if $\B{\lambda}\succeq\left|\B{\tau}-\mathbb{E}_{\up{p}_{\text{te}}}\Phi_{\alpha}(x,y)\right|$, and an approximate upper bound for general $\B{\lambda}$. 
Therefore, the proposed selection method in the paper uses the value of $D$ that results in the lowest upper bound over a range of values for $D$. 
Appendix \ref{Appendix_5:Experiments} further illustrates the adequacy of such approach in practice.

The second hyperparameter $\B{\lambda}$ is determined solving
\begin{align}\label{eq_5:optprob_lambda}
    \min_{\up{p},\B{\lambda}}\  \  &  \mathbf{1}^\text{T}\B{\lambda}\nonumber\\
   \mbox{s.t.}\  \  & \B{\tau}-\B{\lambda}\preceq\sum_{i=1}^{t}\sum_{y\in\set{Y}}\up{p}(y|x_{n+i})\Phi_\alpha(x_{n+i},y)\preceq\B{\tau}+\B{\lambda}\notag\\
    & \B{\lambda},\B{\up{p}}\succeq\B{0}\notag\\
    & \sum_{y\in\set{Y}}\up{p}(y|x_{n+i})=1/t\text{ for }i=1,\ldots,t
\vspace{-0.1cm}
\end{align}
that ensures the uncertainty set used is non-empty.

\textbf{Complexity and implementation without testing instances.}
The computational complexity of the methods proposed is similar to existing methods for covariate shift adaptation. 
Specifically, the step for \ac{DW-KMM} that obtains weights has a similar complexity as that for conventional \ac{KMM}. 
The main difference is that \eqref{eq_4:DW-KMMempirical} has $t$ additional variables and $t+1$ additional constraints corresponding to the weights $\B{\alpha}$. 
The step that obtains the classifier parameters solving convex optimization problem \eqref{eq_3:CSMRC} has the same complexity as that for conventional methods. 
Finally, the step that determines hyperparameters not only avoids the usage of cross-validation but can also reduce complexity. 
In particular, cross-validation with $P$ partitions would require solving \eqref{eq_3:CSMRC} $P$ times for each candidate value for hyperparameters, while the methods proposed only require solving \eqref{eq_5:optprob_lambda} and \eqref{eq_3:CSMRC} once, for each candidate value of $D$.

Algorithm~\ref{alg:methodology} details the implementation of \mbox{DW-GCS} in cases where testing instances are available at training.
The methods proposed can be implemented with small modifications in cases where only training instances are available and the marginals (or their ratios) are known. 
In these cases, weights $\alpha(x)$ and $\beta(x)$ can be determined using \eqref{eq_3:alphabeta_sol} with $C=B/\sqrt{D}$ instead of solving \eqref{eq_4:DW-KMMempirical}, and optimization \eqref{eq_3:CSMRC} can be addressed using the training instances instead of testing instances making use of equality \eqref{eq_2:reweighted_expectation}.
\vspace{-0.2cm}
\section{Experiments}\label{Section_6:Experiments}
\vspace{-0.1cm}
This section shows experimental results for the proposed approach in comparison with existing methods on synthetic and real datasets. 
Reweighted and robust approaches are implemented as in \cite{sugiyama2007,Liu2014} and described in Appendix~\ref{Appendix_1:Reweighted_Robust}, the flattening method is implemented as in \cite{Shimodaira2000}, the RuLSIF is implemented as in \cite{Yamada2011}, the \ac{KMM} method is implemented as in \cite{Huang2006}, and the methods proposed are implemented as described in Alg.~\ref{alg:methodology}. 
The source code for the methods presented is publicly available in the library \mbox{MRCpy} \cite{Bondugula2023} and the experimental setup in \url{https://github.com/MachineLearningBCAM/MRCs-for-Covariate-Shift-Adaptation}.

For existing methods, the regularization parameter has been fine-tuned as shown in Appendix~\ref{Appendix_5:Experiments}.  
For the proposed methods, hyperparameters are obtained as described in Section~\ref{Section_5}. 
The results in this section are complemented by those in Appendix~\ref{Appendix_5:Experiments} that provide further implementation details and experimental results. 
In particular, the appendix shows that selecting the hyperparameter $D$ with lowest minimax risk results in performances near those obtained with the best value for $D$ by grid search.

\begin{figure}[ht]
    \centering
    \psfrag{A}[][][0.7]{Parameter $\delta$}
    \psfrag{B}[][][0.7]{Classification error}
    \psfrag{A_____________}[Bl][Bl][0.7]{ No adapt.}
    \psfrag{_A____________}[Bl][Bl][0.7]{ Reweighted}
    \psfrag{__A___________}[Bl][Bl][0.7]{ Robust}
    \psfrag{___A__________}[Bl][Bl][0.7]{ DW-GCS}
    \psfrag{0}[][][0.5]{0}
    \psfrag{0.1}[][][0.5]{0.1}
    \psfrag{0.2}[][][0.5]{}
    \psfrag{0.3}[][][0.5]{0.3}
    \psfrag{0.4}[][][0.5]{}
    \psfrag{0.5}[][][0.5]{0.5}
    \psfrag{0.6}[][][0.5]{}
    \psfrag{0.7}[][][0.5]{}
    \psfrag{0.8}[][][0.5]{}
    \psfrag{0.9}[][][0.5]{}
    \psfrag{1}[][][0.5]{}
    \psfrag{2}[][][0.5]{0.05}
    \psfrag{4}[][][0.5]{0.1}
    \psfrag{6}[][][0.5]{0.2}
    \psfrag{8}[][][0.5]{0.35}
    \psfrag{10}[][][0.5]{0.4}
    \psfrag{12}[][][0.5]{0.45}
    \includegraphics[width=0.5\textwidth]{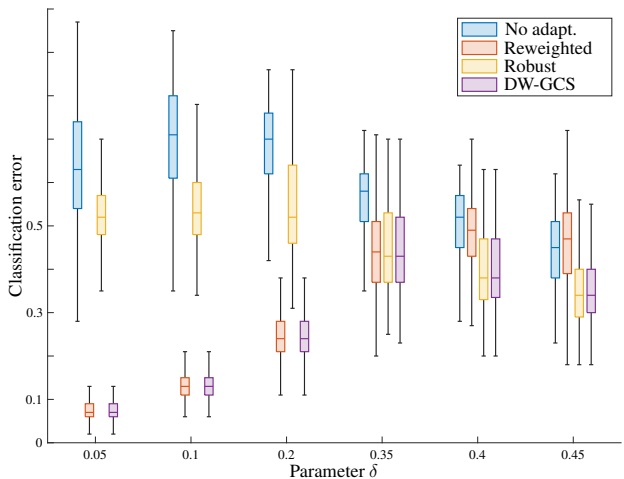}
    \caption{Classification error for different types of covariate shift. In the case $\delta=0.05$, the training support contains that at testing, while in the case $\delta=0.45$ we have the opposite.}
    \label{Synthetic_Experiment}
    \vspace{-5pt}
\end{figure}

\textbf{Experiments with synthetic data.}
In the first set of results we show how the proposed approach can achieve covariate shift adaptation in situations where existing methods are challenged.
For such results, the training and testing samples are drawn from distributions
\vspace{-0.1cm}
\begin{alignat}{1}
    \up{p}_{\text{tr}}(x)&=(0.5-\delta)N(\V{m}_1,\B{\Sigma}_1)+(0.5+\delta)N(\V{m}_2,\B{\Sigma}_2)\notag\\
    \up{p}_{\text{te}}(x)&=(1-\delta)N(\V{m}_1,\B{\Sigma}_1)+\delta N(\V{m}_2,\B{\Sigma}_2)
 \end{alignat}
with $\V{m}_1=[-3/2,0]^\text{T}$, $\V{m}_2=[3/2,0]^{\text{T}}$, $\B{\Sigma}_1=\B{\Sigma}_2=(1/4)\V{I}$, and labels are $y=1$ if $x^{(1)}x^{(2)}\geq 0$ and $y=2$ otherwise. 
We use values \mbox{$\delta\in\{0.05,0.1,0.2,0.35,0.4,0.45\}$} to simulate different relations between the marginals of training and testing instances. 
We utilize the non-linear feature mapping given by instances components and their products and implement existing and proposed methods using the exact marginals. 
In addition, for each type of covariate shift (value of $\delta$) we carry out 1,000 random repetitions with 100 training and testing samples.

\textbf{Results.} 
Figure~\ref{Synthetic_Experiment} shows box-plots corresponding to the classification error of exiting and proposed approaches in comparison to that obtained without covariate shift adaptation ($\alpha(x)=\beta(x)=1$).
The results in the figure show how reweighted (resp. robust) methods obtain poor performances in situations where the support of training (resp. testing) instances does not contain that of testing (resp. training) instances.
On the other hand, the methods proposed can leverage the presented double-weighting approach and adapt to more general covariate shifts. 

\begin{table*}[ht]
\small
\centering
\caption{Classification errors in 21 scenarios show that the proposed methods can more adequately adapt to general covariate shift.  Values in bold show best classification error in each scenario.}
\vspace{0.1cm}
\begin{tabular}{lccccccc}
\toprule
Datasets & Reweighted & Flattening & RuLSIF & Robust & KMM & DW-GCS 0-1 & DW-GCS log \\ 
\midrule
\textbf{Blood}&&&&&&&\\
Feature 1 & $.55\pm.08$ & $.48\pm.11$ & $\V{.29\pm.04}$ & $.34\pm.06$ & $.32\pm.03$ & $.30\pm.03$ & $.31\pm.03$ \\ 
Feature 2 & $.39\pm.03$ & $.38\pm.03$ & $.39\pm.03$ & $.40\pm.03$ & $.39\pm.04$ & $\V{.38\pm.05}$ & $\V{.38\pm.05}$ \\ 
Feature 3 & $.43\pm.05$ & $.41\pm.05$ & $.36\pm.04$ & $.39\pm.04$ & $.36\pm.04$ & $\V{.34\pm.03}$ & $.35\pm.03$ \\ 
PCA & $.48\pm.05$ & $.48\pm.05$ & $.29\pm.05$ & $.44\pm.05$ & $.30\pm.05$ & $\V{.28\pm.04}$ & $\V{.28\pm.04}$ \\ 
\hline 
\noalign{\vspace{0.022cm}} 
\textbf{BreastCancer}&&&&&&&\\
Feature 1 & $.05\pm.02$ & $.05\pm.03$ & $.05\pm.02$ & $.06\pm.03$ & $.06\pm.02$ & $\V{.04\pm.02}$ & $\V{.04\pm.02}$ \\ 
Feature 2 & $.06\pm.02$ & $.05\pm.02$ & $.06\pm.03$ & $.07\pm.03$ & $.06\pm.03$ & $\V{.04\pm.02}$ & $\V{.04\pm.02}$ \\ 
Feature 3 & $.05\pm.02$ & $.05\pm.02$ & $.05\pm.02$ & $.06\pm.03$ & $.05\pm.02$ & $\V{.04\pm.02}$ & $\V{.04\pm.02}$ \\ 
PCA & $.03\pm.01$ & $.03\pm.01$ & $.03\pm.01$ & $.03\pm.01$ & $.03\pm.01$ & $\V{.02\pm.01}$ & $\V{.02\pm.01}$ \\ 
\hline 
\noalign{\vspace{0.022cm}} 
\textbf{Haberman}&&&&&&&\\
Feature 1 & $.48\pm.07$ & $.47\pm.08$ & $.31\pm.06$ & $.41\pm.09$ & $.34\pm.10$ & $\V{.28\pm.07}$ & $.29\pm.06$ \\ 
Feature 2 & $.46\pm.08$ & $.44\pm.08$ & $.31\pm.06$ & $.39\pm.08$ & $.36\pm.10$ & $\V{.29\pm.08}$ & $.30\pm.07$ \\ 
Feature 3 & $\V{.33\pm.05}$ & $\V{.33\pm.05}$ & $\V{.33\pm.05}$ & $.36\pm.06$ & $.42\pm.08$ & $.35\pm.07$ & $.36\pm.06$ \\ 
PCA & $.43\pm.12$ & $.42\pm.12$ & $\V{.29\pm.05}$ & $.42\pm.11$ & $.35\pm.08$ & $.30\pm.08$ & $.31\pm.07$ \\ 
\hline 
\noalign{\vspace{0.022cm}} 
\textbf{Ringnorm}&&&&&&\\
Feature 1 & $.27\pm.02$ & $.26\pm.02$ & $\V{.25\pm.02}$ & $.26\pm.02$ & $\V{.25\pm.02}$ & $\V{.25\pm.02}$ & $\V{.25\pm.02}$ \\ 
Feature 2 & $.28\pm.02$ & $.27\pm.02$ & $\V{.25\pm.02}$ & $.27\pm.02$ & $.26\pm.02$ & $\V{.25\pm.02}$ & $\V{.25\pm.02}$ \\ 
Feature 3 & $.28\pm.02$ & $.27\pm.02$ & $\V{.25\pm.02}$ & $.27\pm.02$ & $.26\pm.03$ & $\V{.25\pm.02}$ & $\V{.25\pm.02}$ \\ 
PCA & $.32\pm.03$ & $.29\pm.03$ & $\V{.25\pm.02}$ & $.26\pm.02$ & $.28\pm.02$ & $.27\pm.02$ & $.26\pm.02$ \\ 
\hline
\noalign{\vspace{0.022cm}} 
\textbf{20 Newsgroups}&&&&&&\\
comp vs sci & $.41\pm.02$ & $.41\pm.02$ & $.41\pm.02$ & $.42\pm.03$ & $.40\pm.02$ & $\V{.22\pm.02}$ & $\V{.22\pm.02}$ \\ 
comp vs talk & $.37\pm.03$ & $.37\pm.03$ & $.37\pm.03$ & $.40\pm.05$ & $.34\pm.03$ & $\V{.11\pm.02}$ & $\V{.11\pm.02}$ \\
rec vs sci & $.43\pm.02$ & $.42\pm.02$ & $.42\pm.02$ & $.42\pm.03$ & $.41\pm.02$ & $\V{.17\pm.02}$ & $\V{.17\pm.02}$ \\ 
rec vs talk & $.40\pm.03$ & $.40\pm.03$ & $.40\pm.03$ & $.41\pm.03$ & $.38\pm.03$ & $\V{.15\pm.02}$ & $\V{.15\pm.02}$ \\  
sci vs talk & $.41\pm.03$ & $.41\pm.02$ & $.41\pm.02$ & $.41\pm.04$ & $.39\pm.02$ & $\V{.20\pm.02}$ & $\V{.20\pm.02}$ \\ 
\bottomrule
\end{tabular}
\label{Table1:Experiment_1}
\vspace{-0.2cm}
\end{table*}

\textbf{Experiments with real datasets.}
In the second set of results, we assess the performance of the proposed methods in comparison with existing techniques using real datasets. 
In particular, reweighted and robust methods are implemented with marginal distributions estimated using log-linear models as shown in \cite{Bickel2007,Bickel2009}. 

We generate covariate shift in the datasets following \cite{Huang2006} and \cite{Gretton2009}. 
In particular, we select training and testing samples with different probabilities based on the medians of the first 3 features, and based on the median of the first principal component of features.
In \cite{Huang2006} and \cite{Gretton2009}, covariate shift is generated with a biased sampling for testing instances that are drawn with probability $\delta_{\text{te}}$ if the first principal component or feature is larger than a certain value. 
In those works, the training samples are uniformly sampled, so that the generated covariate shifts correspond to situations where the support of training samples contains that of testing samples. 
In the numerical results of the table below, we generate more general covariate shifts by using a biased sampling both for training and testing instances (using probabilities $\delta_{\text{tr}}=0.7$ and $\delta_{\text{te}}=0.3$). 
These covariate shifts correspond to situations where the support of training and testing samples have certain overlap but they do not need to be contained in each other. Additionally, we include experimental results using the “News20groups” dataset that is intrinsically affected by a covariate shift since the training and testing partitions correspond to different times \cite{Zhang2013}.
We consider the same 5 binary problems used in \cite{Zhang2013}, utilize the 1,000 features with highest Pearson’s correlation, and randomly sample 1,000 training and testing samples in each repetition.

\textbf{Results.}
Table~\ref{Table1:Experiment_1} shows the averaged classification error corresponding to different datasets and covariate shift situations, together with their standard deviations over 100 random partitions as detailed in Appendix~\ref{Appendix_5:Experiments}. 
The first column of the table describes the different covariate shift datasets generated as described above. 

Overall, the experimental results show that the proposed method provides improved adaptation to general covariate shifts, even in situations where the supports of training and testing samples are not contained in each other. 
These results agree with the discussion in Sections~\ref{Section_2.1} and \ref{Section_3.1} as well as the theoretical results in Corollary~\ref{cor3.3} and Theorem~\ref{th_4:bound_GKMM} that show how the proposed methodology can be effective in situations where existing methods based on single weights are challenged. 
The improvement obtained by the methods presented can be clearly observed by comparing the results obtained by the \ac{KMM} method, since that technique is the most closely related to the proposed method. 
In particular, the results show that significant performance improvements can be obtained using a double weighting of both training and testing samples solving \eqref{eq_4:DW-KMMempirical} instead of using the existing \ac{KMM} method (that solves \eqref{eq_4:DW-KMMempirical} fixing the weights $\alpha$ to be one).
\vspace{-0.2cm}
\section{Conclusion}
\vspace{-0.1cm}
Existing approaches for covariate shift adaptation use the ratios between marginal distributions to either weight training or testing samples.
However, the performance of such approaches can be poor when the marginals' supports are not contained in each other or when marginals' ratios take large values. 
This paper proposes a minimax risk classification (\ac{MRC}) approach for covariate shift adaptation that avoids such limitations by weighting both training and testing samples. We present effective techniques that obtain both sets of weights generalizing the conventional kernel mean matching method that only obtains weights for training samples. 
In addition, we present generalization bounds for the proposed methods that show a significant increase in effective sample size. 
The unifying approach and the learning methods proposed can enable techniques capable to adapt to more general scenarios affected by covariate shift.

\vspace{-0.2cm}
\section*{Acknowledgments}
\vspace{-0.1cm}
Funding in direct support of this work has been provided by projects PID2019-105058GA-I00, CNS2022-135203, and CEX2021-001142-S funded by MCIN/AEI/10.13039/ 501100011033  and the European Union “NextGenerationEU”/PRTR, programmes ELKARTEK and BERC-2022-2025 funded by the Basque Government, the project “Early Prognosis of COVID-19 Infections via Machine Learning” funded by the AXA Research Fund, and by the JHU-Amazon AI2AI Faculty Award.

\newpage
\bibliography{bibliography.bib}

\begin{thebibliography}{34}
\providecommand{\natexlab}[1]{#1}
\providecommand{\url}[1]{\texttt{#1}}
\expandafter\ifx\csname urlstyle\endcsname\relax
  \providecommand{\doi}[1]{doi: #1}\else
  \providecommand{\doi}{doi: \begingroup \urlstyle{rm}\Url}\fi

\bibitem[Altun \& Smola(2006)Altun and Smola]{Altun2006}
Altun, Y. and Smola, A.
\newblock Unifying divergence minimization and statistical inference via convex
  duality.
\newblock In \emph{Proceedings of the 19th Annual Conference on Computational
  Learning Theory}, pp.\  139 -- 153, 2006.

\bibitem[Bickel et~al.(2007)Bickel, Brückner, and Scheffer]{Bickel2007}
Bickel, S., Brückner, M., and Scheffer, T.
\newblock Discriminative learning for differing training and test
  distributions.
\newblock In \emph{Proceedings of the 24th International Conference on Machine
  Learning}, pp.\  81 -- 88, 2007.

\bibitem[Bickel et~al.(2009)Bickel, Brückner, and Scheffer]{Bickel2009}
Bickel, S., Brückner, M., and Scheffer, T.
\newblock Discriminative learning under covariate shift.
\newblock \emph{Journal of Machine Learning Research}, 10:\penalty0 2137 –
  2155, 2009.

\bibitem[Bondugula et~al.(2023)Bondugula, Alvarez, Segovia-Martín, Pérez, and
  Mazuelas]{Bondugula2023}
Bondugula, K., Alvarez, V., Segovia-Martín, J.~I., Pérez, A., and Mazuelas,
  S.
\newblock {MRC}py: A library for minimax risk classifiers.
\newblock \emph{arXiv preprint arXiv:2108.01952}, 2023.

\bibitem[Boucheron et~al.(2013)Boucheron, Lugosi, and Massart]{Boucheron2013}
Boucheron, S., Lugosi, G., and Massart, P.
\newblock \emph{Concentration Inequalities: A Nonasymptotic Theory of
  Independence}.
\newblock Oxford University Press, 2013.

\bibitem[Chen et~al.(2016)Chen, Monfort, Liu, and Ziebart]{Chen2016}
Chen, X., Monfort, M., Liu, A., and Ziebart, B.~D.
\newblock Robust covariate shift regression.
\newblock In \emph{Proceedings of the 19th International Conference on
  Artificial Intelligence and Statistics}, pp.\  1270 – 1279, 2016.

\bibitem[Cortes \& Mohri(2014)Cortes and Mohri]{Cortes2014}
Cortes, C. and Mohri, M.
\newblock Domain adaptation and sample bias correction theory and algorithm for
  regression.
\newblock \emph{Theoretical Computer Science}, 519:\penalty0 103 -- 126, 2014.

\bibitem[Cortes et~al.(2008)Cortes, Mohri, Riley, and Rostamizadeh]{Cortes2008}
Cortes, C., Mohri, M., Riley, M., and Rostamizadeh, A.
\newblock Sample selection bias correction theory.
\newblock In \emph{Proceedings of the 19th International Conference on
  Algorithmic Learning Theory}, pp.\  38 -- 53, 2008.

\bibitem[Cortes et~al.(2010)Cortes, Mansour, and Mohri]{Cortes2010}
Cortes, C., Mansour, Y., and Mohri, M.
\newblock Learning bounds for importance weighting.
\newblock In \emph{Proceedings of the 24th Annual Conference on Neural
  Information Processing Systems}, pp.\  442 -- 450, 2010.

\bibitem[Dua \& Graff(2017)Dua and Graff]{Dua2019}
Dua, D. and Graff, C.
\newblock {UCI} {M}achine {L}earning {R}epository, 2017.
\newblock URL \url{http://archive.ics.uci.edu/ml}.

\bibitem[Dudík et~al.(2005)Dudík, Schapire, and Phillips]{Dudik2005}
Dudík, M., Schapire, R.~E., and Phillips, S.~J.
\newblock Correcting sample selection bias in maximum entropy density
  estimation.
\newblock In \emph{Proceedings of the 19th Annual Conference on Neural
  Information Processing Systems}, pp.\  323 – 330, 2005.

\bibitem[Farnia \& Tse(2016)Farnia and Tse]{Farnia2016}
Farnia, F. and Tse, D.
\newblock A minimax approach to supervised learning.
\newblock In \emph{Proceedings of the 30th Annual Conference on Neural
  Information Processing Systems}, pp.\  4240 -- 4248, 2016.

\bibitem[Fathony et~al.(2016)Fathony, Liu, Asif, and Ziebart]{Fathony2016}
Fathony, R., Liu, A., Asif, K., and Ziebart, B.~D.
\newblock Adversarial multiclass classification: A risk minimization
  perspective.
\newblock In \emph{Proceedings of the 30th Annual Conference on Neural
  Information Processing Systems}, pp.\  559 -- 567, 2016.

\bibitem[Gretton et~al.(2008)Gretton, Smola, Huang, Schmittfull, Borgwardt, and
  Sch{\"o}lkopf]{Gretton2009}
Gretton, A., Smola, A., Huang, J., Schmittfull, M., Borgwardt, K., and
  Sch{\"o}lkopf, B.
\newblock Covariate shift by kernel mean matching.
\newblock In \emph{Dataset shift in machine learning}, pp.\  131 -- 160. MIT
  Press, 2008.

\bibitem[Huang et~al.(2006)Huang, Smola, Gretton, Borgwardt, and
  Schölkopf]{Huang2006}
Huang, J., Smola, A.~J., Gretton, A., Borgwardt, K.~M., and Schölkopf, B.
\newblock Correcting sample selection bias by unlabeled data.
\newblock In \emph{Proceedings of the 20th Annual Conference on Neural
  Information Processing Systems}, pp.\  601 -- 608, 2006.

\bibitem[Kanamori et~al.(2009)Kanamori, Hido, and Sugiyama]{Kanamori2009}
Kanamori, T., Hido, S., and Sugiyama, M.
\newblock A least-squares approach to direct importance estimation.
\newblock \emph{Journal of Machine Learning Research}, 10:\penalty0 1391 –
  1445, 2009.

\bibitem[Lin et~al.(2002)Lin, Lee, and Wahba]{Lin2002}
Lin, Y., Lee, Y., and Wahba, G.
\newblock Support vector machines for classification in nonstandard situations.
\newblock \emph{Machine Learning}, 46\penalty0 (1):\penalty0 191 – 202, 2002.

\bibitem[Liu \& Ziebart(2014)Liu and Ziebart]{Liu2014}
Liu, A. and Ziebart, B.~D.
\newblock Robust classification under sample selection bias.
\newblock In \emph{Proceedings of the 28th Annual Conference on Neural
  Information Processing Systems}, pp.\  37 -- 45, 2014.

\bibitem[Liu \& Ziebart(2017)Liu and Ziebart]{Liu2017}
Liu, A. and Ziebart, B.~D.
\newblock Robust covariate shift prediction with general losses and feature
  views.
\newblock \emph{arXiv preprint arXiv:1712.10043}, 2017.

\bibitem[Liu et~al.(2013)Liu, Yamada, Collier, and Sugiyama]{Liu2013}
Liu, S., Yamada, M., Collier, N., and Sugiyama, M.
\newblock Change-point detection in time-series data by relative density-ratio
  estimation.
\newblock \emph{Neural Networks}, 43:\penalty0 72 – 83, 2013.

\bibitem[Mazaheri et~al.(2020)Mazaheri, Jain, and Bruck]{Mazaheri2020}
Mazaheri, B., Jain, S., and Bruck, J.
\newblock Robust correction of sampling bias using cumulative distribution
  functions.
\newblock In \emph{Proceedings of the 34th Annual Conference on Neural
  Information Processing Systems}, pp.\  3546 -- 3556, 2020.

\bibitem[Mazuelas et~al.(2022)Mazuelas, Shen, and
  Perez]{Mazuelas2022generalized}
Mazuelas, S., Shen, Y., and Perez, A.
\newblock Generalized maximum entropy for supervised classification.
\newblock \emph{IEEE Transactions on Information Theory}, 68\penalty0
  (4):\penalty0 2530--2550, 2022.

\bibitem[Mazuelas et~al.(2023)Mazuelas, Romero, and
  Grünwald]{Mazuelas2023minimax}
Mazuelas, S., Romero, M., and Grünwald, P.
\newblock Minimax risk classifiers with 0-1 loss.
\newblock \emph{arXiv preprint arXiv:2201.06487}, 2023.

\bibitem[Quiñonero-Candela et~al.(2008)Quiñonero-Candela, Sugiyama,
  Schwaighofer, and Lawrence]{Quinonero2008}
Quiñonero-Candela, J., Sugiyama, M., Schwaighofer, A., and Lawrence, N.~D.
\newblock \emph{Dataset Shift in Machine Learning}.
\newblock MIT Press, 2008.

\bibitem[Reddi et~al.(2015)Reddi, Póczos, and Smola]{Reddi2015}
Reddi, S.~J., Póczos, B., and Smola, A.
\newblock Doubly robust covariate shift correction.
\newblock In \emph{Proceedings of the 29th AAAI Conference on Artificial
  Intelligence}, pp.\  2949 – 2955, 2015.

\bibitem[Shimodaira(2000)]{Shimodaira2000}
Shimodaira, H.
\newblock Improving predictive inference under covariate shift by weighting the
  log-likelihood function.
\newblock \emph{Journal of Statistical Planning and Inference}, 90\penalty0
  (2):\penalty0 227 -- 244, 2000.

\bibitem[Sugiyama \& Kawanabe(2012)Sugiyama and Kawanabe]{Sugiyama2012}
Sugiyama, M. and Kawanabe, M.
\newblock \emph{Machine learning in non-stationary environments: Introduction
  to covariate shift adaptation}.
\newblock MIT press, 2012.

\bibitem[Sugiyama et~al.(2007)Sugiyama, Krauledat, and Müller]{sugiyama2007}
Sugiyama, M., Krauledat, M., and Müller, K.-R.
\newblock Covariate shift adaptation by importance weighted cross validation.
\newblock \emph{Journal of Machine Learning Research}, 8:\penalty0 985 –
  1005, 2007.

\bibitem[Tsuboi et~al.(2009)Tsuboi, Kashima, Hido, Bickel, and
  Sugiyama]{Tsuboi2009}
Tsuboi, Y., Kashima, H., Hido, S., Bickel, S., and Sugiyama, M.
\newblock Direct density ratio estimation for large-scale covariate shift
  adaptation.
\newblock \emph{Journal of Information Processing}, 17:\penalty0 138 -- 155,
  2009.

\bibitem[Wen et~al.(2014)Wen, Yu, and Greiner]{Wen2014}
Wen, J., Yu, C.-N., and Greiner, R.
\newblock Robust learning under uncertain test distributions: Relating
  covariate shift to model misspecification.
\newblock In \emph{Proceedings of the 31st International Conference on Machine
  Learning}, pp.\  631 -- 639, 2014.

\bibitem[Yamada et~al.(2011)Yamada, Suzuki, Kanamori, Hachiya, and
  Sugiyama]{Yamada2011}
Yamada, M., Suzuki, T., Kanamori, T., Hachiya, H., and Sugiyama, M.
\newblock Relative density-ratio estimation for robust distribution comparison.
\newblock In \emph{Proceedings of the 25th Annual Conference on Neural
  Information Processing Systems}, pp.\  594 – 602, 2011.

\bibitem[Yu \& Szepesvári(2012)Yu and Szepesvári]{Yu2012}
Yu, Y.-L. and Szepesvári, C.
\newblock Analysis of kernel mean matching under covariate shift.
\newblock In \emph{Proceedings of the 29th International Conference on Machine
  Learning}, pp.\  1147 – 1154, 2012.

\bibitem[Zadrozny(2004)]{Zadrozny2004}
Zadrozny, B.
\newblock Learning and evaluating classifiers under sample selection bias.
\newblock In \emph{Proceedings of the 21st International Conference on Machine
  Learning}, pp.\  114, 2004.

\bibitem[Zhang et~al.(2013)Zhang, Zheng, Wang, Kwok, Yang, and
  Marsic]{Zhang2013}
Zhang, K., Zheng, V.~W., Wang, Q., Kwok, J.~T., Yang, Q., and Marsic, I.
\newblock Covariate shift in hilbert space: A solution via surrogate kernels.
\newblock In \emph{Proceedings of the 30th International Conference on Machine
  Learning}, pp.\  388 -- 395, 2013.

\end{thebibliography}

\newpage
\appendix
\onecolumn

\section{Detailed derivations describing existing methods and relation with the proposed framework}\label{Appendix_1:Reweighted_Robust}
The following describes reweighted and robust methods for binary classification with $\set{Y}\in\left\lbrace-1,1\right\rbrace$ and log-loss. In particular, we show how, using the specific weights in \eqref{eq_2:reweighted_expectation} and \eqref{eq_2:robust_expectation}, such methods can be obtained from Theorem~\ref{th_3:MRC} in Section~\ref{framework} corresponding to the proposed framework.

Reweighted methods consider classification rules of the form
\begin{equation}\label{eq_2:reweighted_classification_rule}
	\up{h}(y|x)=\frac{1 }{1+\exp\left\lbrace -y\V{x}^T\B{\mu}\right\rbrace }
\end{equation}
and learn the parameter $\B{\mu}$ using the fact that equality \eqref{eq_2:reweighted_expectation} in Section~\ref{Section_2.1} allows to estimate expected losses with respect to the test distribution using training samples since
\begin{alignat}{1}\label{eq_2:reweighted_LR_transform}
		\mathbb{E}_{\up{p}_{\text{te}}(x,y)}\log\left(1+\exp\left\lbrace-y\V{x}^T\B{\mu} \right\rbrace  \right) 
		= \mathbb{E}_{\up{p}_{\text{tr}}(x,y)}\beta(x)\log\left(1+\exp\left\lbrace-y\V{x}^T\B{\mu} \right\rbrace  \right).\notag
\end{alignat}
for $\beta(x)=\up{p}_{\text{te}}(x)/\up{p}_{\text{tr}}(x)$.

Robust methods consider classification rules of the form
\begin{equation}\label{eq_2:robust_classification_rule}
	\up{h}(y|x)=\frac{1 }{1+\exp\left\lbrace -\alpha(x)y\V{x}^T\B{\mu}\right\rbrace }
\end{equation}
with $\alpha(x)=\up{p}_{\text{tr}}(x)/\up{p}_{\text{te}}(x)$. 
Such methods learn the parameter $\B{\mu}$ using the fact that equality \eqref{eq_2:robust_expectation} in Section~\ref{Section_2.1} allows to estimate the expected gradient of losses with respect to the test distribution using training samples since
\begin{align*}\label{eq_2:reweighted_gradient_transform}
		\mathbb{E}_{\up{p}_{\text{te}}(x,y)}\nabla_{\B{\mu}}\log\Big(1+\exp&\left\lbrace-\frac{\up{p}_{\text{tr}}(x)}{\up{p}_{\text{te}}(x)}y\V{x}^T\B{\mu} \right\rbrace  \Big) \\
		&=\mathbb{E}_{\up{p}_{\text{te}}(x,y)}\alpha(x)\left(\frac{-y\V{x}^T}{1+\exp\left\lbrace\frac{\up{p}_{\text{tr}}(x)}{\up{p}_{\text{te}}(x)}y\V{x}^T\B{\mu} \right\rbrace} \right)
		= \mathbb{E}_{\up{p}_{\text{tr}}(x,y)}\frac{-y\V{x}^T}{1+\exp\left\lbrace\frac{\up{p}_{\text{tr}}(x)}{\up{p}_\text{te}(x)}y\V{x}^T\B{\mu} \right\rbrace}.
\end{align*}
for $\alpha(x)=\up{p}_{\text{tr}}(x)/\up{p}_{\text{te}}(x)$.

For their derivation from the Theorem~\ref{th_3:MRC} corresponding to the proposed framework; taking $\Phi(x,y)=y\V{x}/2$, we have that optimization problem in \eqref{eq_3:CSMRC} of Theorem~\ref{th_3:MRC} becomes
\begin{equation}\label{eq_3.1:log_regression}
   \min_{\B{\mu}}-\frac{1}{n}\sum_{i=1}^n\beta(x_i)\frac{y_i\V{x}_i^T}{2}\B{\mu}+\mathbb{E}_{\up{p}_{\text{te}}(x)}\Big\{\log \big(\exp\big\{\alpha(x)\frac{\V{x}^T}{2}\B{\mu}\big\}+\exp\big\{-\alpha(x)\frac{\V{x}^T}{2}\B{\mu}\big\}\big)\Big\}
\end{equation}
in binary classification with log-loss.

If $\alpha(x)=1$, $\beta(x)=\up{p}_{\text{te}}(x)/\up{p}_{\text{tr}}(x)$, the classifier in \eqref{eq_3:h_equality_log} of Theorem~\ref{th_3:MRC} coincides with that of reweighted methods in \eqref{eq_2:reweighted_classification_rule}. In addition, using \eqref{eq_2:reweighted_expectation} in Section~\ref{Section_2.1} and approximating the expectation with training samples, the optimization in \eqref{eq_3.1:log_regression} becomes 
\begin{equation}
    -\frac{1}{n}\sum_{i=1}^n\frac{\up{p}_{\text{te}}(x_i)}{\up{p}_{\text{tr}}(x_i)}\log\left(1+\exp\left\lbrace-y_i\V{x}_i^T\B{\mu}\right\rbrace\right)
\end{equation}
 that coincides with that of reweighted logistic regression \citep{Sugiyama2012}.

If $\alpha(x)=\up{p}_{\text{tr}}(x)/\up{p}_{\text{te}}(x)$, $\beta(x)=1$, the classifier in \eqref{eq_3:h_equality_log} of Theorem~\ref{th_3:MRC} coincides with that of robust methods in \eqref{eq_2:robust_classification_rule}. In addition, using \eqref{eq_2:robust_expectation} in Section~\ref{Section_2.1}, the gradient of objective function in \eqref{eq_3.1:log_regression} becomes 
\begin{equation}\label{eq_3.1:gradient}
    -\frac{1}{n}\sum_{i=1}^n\frac{\V{x}_i^Ty_i}{2}+\mathbb{E}_{\up{p}_{\text{tr}}(x)}\frac{\V{x}^T}{2}\frac{1-\exp\left\lbrace-\frac{\up{p}_{\text{tr}}(x)}{\up{p}_{\text{te}}(x)}\V{x}^T\B{\mu}\right\rbrace}{1+\exp\left\lbrace-\frac{\up{p}_{\text{tr}}(x)}{\up{p}_{\text{te}}(x)}\V{x}^T\B{\mu}\right\rbrace}
\end{equation}
that coincides with that shown in equation (7) in \citep{Liu2014} for robust methods.

\section{Proofs for Section \ref{Section_3}}\label{Appendix_2:Proofs_Sect3}

The proofs of Theorem \ref{th_3:MRC} and \ref{th_3:generalization_bound} below are done for the case of finite $\set{X}$. The proofs for infinite $\set{X}$ can be carried out analogously using Fenchel duality instead of Lagrange duality, similarly to as is done in \cite{Altun2006,Mazuelas2023minimax}.

\begin{proof}[\textbf{Proof of Theorem \ref{th_3:MRC}.}]

Firstly, for each $\up{h}\in\text{T}(\set{X},\set{Y})$, we have that 
\begin{align}\label{opt11}\begin{array}{ccl}\max_{\up{p}\in\set{U}} \ell(\up{h},\up{p})=&\underset{\V{p}}{\max} &\V{l}^{\text{T}}\V{p}-I_+(\V{p})\\
&\mbox{s.t.}&\sum_{y\in\set{Y}}\up{p}(x,y)=\up{p}_{\text{te}}(x),\  \forall x\in\set{X}\\
&&\B{\tau}-\B{\lambda}\preceq\B{\Phi}_\alpha^{\text{T}}\V{p}\preceq\B{\tau}+\B{\lambda}
\end{array}\end{align}
 where $\V{l}$, $\V{p}$, and $\B{\Phi}_\alpha$ denote the vectors and matrix with rows $\ell(\up{h},(x,y))$, $\up{p}(x,y)$, and $\Phi_\alpha(x,y)^{\text{T}}$, respectively, for \mbox{$x\in\set{X}$, $y\in\set{Y}$}, and
$$I_+(\V{p})=\left\{\begin{array}{cc}0&\mbox{if }\V{p}\succeq\V{0}\\\infty&\mbox{otherwise.}\end{array}\right.$$
Optimization problem \eqref{opt11} has Lagrange dual
$$\begin{array}{ccl}&\underset{\B{\mu}_1,\B{\mu}_2,\nu(x)}{\min} &-\big(\B{\tau}-\B{\lambda}\big)^{\text{T}}\B{\mu}_1
+\big(\B{\tau}+\B{\lambda}\big)^{\text{T}}\B{\mu}_2+\mathbb{E}_{\up{p}_{\text{te}}(x)}\nu(x) +f^*(\B{\Phi}_\alpha(\B{\mu}_1-\B{\mu}_2)-\B{\nu})\\&
\mbox{s.t.}&\B{\mu}_1,\B{\mu}_2\succeq\V{0}\end{array}$$
where $\B{\nu}$ is the vector in $\mathbb{R}^{|\set{X}||\set{Y}|}$ with component corresponding with $(x,y)$ for $x\in\set{X}$, $y\in\set{Y}$ given by $\nu(x)$, and $f^*$ is the conjugate function of $f(\V{p})=-\V{l}^{\text{T}}\V{p}+I_+(\V{p})$ given by
$$f^*(\V{w})=\sup_{\V{p}\succeq\V{0}}\V{w}^{\text{T}}\V{p}+\V{l}^{\text{T}}\V{p}=\left\{\begin{array}{cc}0&\mbox{if }\V{w}\preceq-\V{l}\\\infty&\mbox{otherwise}\end{array}\right..$$
Therefore, the Lagrange dual above becomes
$$\begin{array}{ccl}&\underset{\B{\mu}_1,\B{\mu}_2,\nu(x)}{\min} &-\big(\B{\tau}-\B{\lambda}\big)^{\text{T}}\B{\mu}_1
+\big(\B{\tau}+\B{\lambda}\big)^{\text{T}}\B{\mu}_2+\mathbb{E}_{\up{p}_{\text{te}}(x)}\nu(x)\\&
\mbox{s.t.}&\B{\mu}_1,\B{\mu}_2\succeq\V{0}\\&
&\Phi_\alpha(x,y)^{\text{T}}(\B{\mu}_1-\B{\mu}_2)-\nu(x)\leq-\ell(\up{h},(x,y)),\  \forall x\in\set{X},y\in\set{Y}.\end{array}$$
It is easy to see that the solution of such optimization problem $\bar{\B{\mu}}_1,\bar{\B{\mu}}_2$ satisfies that $\bar{\mu}_{1}^{(i)}\bar{\mu}_{2}^{(i)}=0$ for any $i$ such that $\lambda^{(i)}>0$. Then $\B{\lambda}^{\text{T}}(\bar{\B{\mu}}_1+\bar{\B{\mu}}_2)=\B{\lambda}^{\text{T}}|\bar{\B{\mu}}_1-\bar{\B{\mu}}_2|$ and taking $\B{\mu}=\B{\mu}_1-\B{\mu}_2$ the Lagrange dual above is equivalent to
$$\begin{array}{ccl}&\underset{\B{\mu},\nu(x)}{\min} &-\B{\tau}^{\text{T}}\B{\mu}
+\B{\lambda}^{\text{T}}|\B{\mu}|+\mathbb{E}_{\up{p}_{\text{te}}(x)}\nu(x)\\&
&\Phi_\alpha(x,y)^{\text{T}}\B{\mu}-\nu(x)\leq-\ell(\up{h},(x,y)),\  \forall x\in\set{X},y\in\set{Y}\end{array}$$
that has the same value as $\max_{\up{p}\in\set{U}} \ell(\up{h},\up{p})$ since the constraints in \eqref{opt11} are affine and $\set{U}$ is non-empty. 

Therefore, 
\begin{align*}\min_{\up{h}\in\text{T}(\set{X},\set{Y})}\max_{\up{p}\in\set{U}} \ell(\up{h},\up{p})=\min_{\up{h},\B{\mu},\nu(x)}&\  -\B{\tau}^{\text{T}}\B{\mu}
+\B{\lambda}^{\text{T}}|\B{\mu}|+\mathbb{E}_{\up{p}_{\text{te}}(x)}\nu(x)\\
&\Phi_\alpha(x,y)^{\text{T}}\B{\mu}-\nu(x)\leq-\ell(\up{h},(x,y)),\  \forall x\in\set{X},y\in\set{Y}.\end{align*}
For 0-1-loss we have that
\begin{align*}\Phi_\alpha(x,y)^{\text{T}}\B{\mu}&-\nu(x)\leq -1+\up{h}(y|x),\  \forall x\in\set{X},y\in\set{Y}\\
&\Rightarrow\sum_{y\in\set{C}}\big(\Phi_\alpha(x,y)^{\text{T}}\B{\mu}-\nu(x)+1\big)\leq 1,\  \forall \set{C}\subseteq\set{Y}, x\in\set{X}\\
&\Rightarrow \nu(x)\geq 1+\frac{\sum_{y\in\set{C}}\Phi_\alpha(x,y)^{\text{T}}\B{\mu}-1}{|\set{C}|}, \  \forall \set{C}\subseteq\set{Y}, x\in\set{X}\\
&\Rightarrow \nu(x)\geq\varphi_{01}(\B{\mu},x,\alpha(x)), \  \forall x\in\set{X}.
\end{align*}

Therefore, for each $\B{\mu}$, we have that any classification rule satisfying 
$$\up{h}(y|x)\geq \Phi_\alpha(x,y)^{\text{T}}\B{\mu}-\varphi_{01}(\B{\mu},x,\alpha(x))+1,\  \forall x\in\set{X},y\in\set{Y}$$
is solution of 
\begin{alignat}{2}
\min_{\up{h},\nu(x)}&\  \mathbb{E}_{\up{p}_{\text{te}}(x)}\nu(x)&&=\mathbb{E}_{\up{p}_{\text{te}}(x)}\varphi_{01}(\B{\mu},x,\alpha(x))\notag\\
&\Phi_\alpha(x,y)^{\text{T}}\B{\mu}-\nu(x)+1\leq\up{h}(y|x),\  \forall x\in\set{X},y\in\set{Y}&&\notag\end{alignat}
and the result is obtained because for any $x\in\set{X}$, we have that
$$\sum_{y\in\set{Y}}\big(\Phi_\alpha(x,y)^{\text{T}}\B{\mu}-\varphi_{01}(\B{\mu},x,\alpha(x))+1\big)_+=1$$
because otherwise there would exist $\nu_x<\varphi_{01}(\B{\mu},x,\alpha(x))$ such that
$$1=\sum_{y\in\set{Y}}\big(\Phi_\alpha(x,y)^{\text{T}}\B{\mu}-\nu_x+1\big)_+=\max_{\set{C}\subseteq\set{Y}}\sum_{y\in\set{C}}\big(\Phi_\alpha(x,y)^{\text{T}}\B{\mu}-\nu_x+1\big)$$
which contradicts the definition of $\varphi_{01}(\B{\mu},x,\alpha(x))$.

The case of log-loss is analogous to the case for 0-1-loss above taking into account that
\begin{align*}\Phi_\alpha(x,y)^{\text{T}}\B{\mu}&-\nu(x)\leq \log(\up{h}(y|x)),\  \forall x\in\set{X},y\in\set{Y}\\
&\Rightarrow\sum_{y\in\set{Y}}\exp\{\Phi_\alpha(x,y)^{\text{T}}\B{\mu}-\nu(x)\}\leq 1,\  \forall  x\in\set{X}\\
&\Rightarrow \nu(x)\geq \log\big(\sum_{y\in\set{Y}}\exp\{\Phi_\alpha(x,y)^{\text{T}}\B{\mu}\}\big), \  \forall x\in\set{X}\\
&\Rightarrow \nu(x)\geq\varphi_{\text{log}}(\B{\mu},x,\alpha(x)), \  \forall x\in\set{X}.
\end{align*}
\end{proof}
The lemma below is used in the proof of Theorem~\ref{th_3:generalization_bound}.
\begin{lemma}\label{lemma_1}
	Let $\set{U}$ be the uncertainty set given by \eqref{uncertainty_set} for $\B{\tau},\B{\lambda}\in\mathbb{R}^m$, and $\up{h}$ be a classification rule. If
	\begin{alignat}{1}
		\overline{R}_{01}\left(\set{U},\up{h}\right) &=\min_{\B{\mu}}-\B{\tau}^T\B{\mu}+\mathbb{E}_{\up{p}_{\text{te}}(x)}\max_{y\in\set{Y}}\left\lbrace 1+\alpha(x)\Phi(x,y)^T\B{\mu}-\up{h}(y|x)\right\rbrace+\B{\lambda}^T|\B{\mu}| \\
        \overline{R}_{\log}\left(\set{U},\up{h}\right) &=\min_{\B{\mu}}-\B{\tau}^T\B{\mu}+\mathbb{E}_{\up{p}_{\text{te}}(x)}\max_{y\in\set{Y}}\left\lbrace \alpha(x)\Phi(x,y)^T\B{\mu}-\log\up{h}(y|x)\right\rbrace+\B{\lambda}^T|\B{\mu}| 
	\end{alignat}
	then, for any $\up{p}\in\set{U}$
	\begin{align}
\ell_{01}(\up{h},\up{p})&\leq	\overline{R}_{01}\left(\set{U},\up{h}\right)\\
 	\ell_{\log}(\up{h},\up{p})&\leq	\overline{R}_{\log}\left(\set{U},\up{h}\right).
	\end{align}
\end{lemma}

\begin{proof}[\textbf{Proof of Lemma \ref{lemma_1}}]
	
	The proof is analogous to the proof of Theorem 5 of \citep{Mazuelas2023minimax}. The case $\set{U}=\emptyset$ is trivial. For the case where $\set{U}\neq\emptyset$, we will first calculate the Lagrange dual of the optimization problem $\min_{\hat{\up{p}}\in\set{U}}\mathbb{E}_{\hat{\up{p}}}q$ for a general function $q:\set{X}\times\set{Y}\rightarrow\mathbb{R}$. Then we will consider the fact that for any $\up{p}\in\set{U}$ and $\up{h}\in \text{T}(\set{X},\set{Y})$,
	\begin{equation}\nonumber
		\min_{\hat{\up{p}}\in\set{U}}\ell(\up{h},\hat{\up{p}})\leq \ell(\up{h},\up{p})\leq	\max_{\hat{\up{p}}\in\set{U}}\ell(\up{h},\hat{\up{p}})
	\end{equation}
	and
	\begin{alignat}{1}
			\max_{\hat{\up{p}}\in\set{U}}\ell_{01}(\up{h},\hat{\up{p}})&=-\min_{\hat{\up{p}}\in\set{U}}\mathbb{E}_{\hat{\up{p}}}\left\lbrace \up{h}(y|x)-1 \right\rbrace\nonumber\\
            \max_{\hat{\up{p}}\in\set{U}}\ell_{\log}(\up{h},\hat{\up{p}})&=-\min_{\hat{\up{p}}\in\set{U}}\mathbb{E}_{\hat{\up{p}}}\log\up{h}(y|x)\nonumber
	\end{alignat}
    for 0-1-loss and log-loss respectively.
	
	First, we have that $\min_{\hat{\up{p}}\in\set{U}}\mathbb{E}_{\hat{\up{p}}}q$ is equal to
	\begin{alignat}{2}\label{theorem_5_proof_step_1}
			\min_{\hat{\B{\up{p}}}}\  \  & &&\B{q}^T\hat{\B{\up{p}}}+I_+(\hat{\B{\up{p}}})\notag\\
			\mbox{s.t.}\  \  & &&-\sum_{y\in\set{Y}}\hat{\up{p}}(x,y)=-\up{p}_{\text{te}}(x)\text{ for all }x\in\set{X}\notag\\
			& &&\B{\tau}-\B{\lambda}\preceq\B{\Phi}_{\alpha}^T\hat{\B{\up{p}}}\preceq\B{\tau}+\B{\lambda}
	\end{alignat}
	where $\hat{\B{\up{p}}}$, $\B{q}$, $\B{\Phi}_{\alpha}$ denote the vectors and matrix with rows $\hat{\up{p}}(x,y)$, $q(x,y)$ and $\alpha(x)\Phi(x,y)^T$, respectively, for $ x\in\set{X}$, $y\in\set{Y}$, and
	\begin{equation*}
		I_+(\hat{\B{\up{p}}})=\left\lbrace\begin{array}{cc} 0 & \text{if }\hat{\B{\up{p}}}\succeq\B{0} \\ \infty & \text{otherwise}. \end{array}\right.
	\end{equation*}
	Optimization problem \eqref{theorem_5_proof_step_1} has Lagrange dual
	\begin{equation*}
		\begin{split}
			\max_{\B{\mu}_1,\B{\mu}_2,\nu(x)}\quad & (\B{\tau}-\B{\lambda})^T\B{\mu}_1-(\B{\tau}+\B{\lambda})^T\B{\mu}_2+\mathbb{E}_{\up{p}_{\text{te}}(x)}\nu(x)-f^*\left(\B{\Phi}_{\alpha}(\B{\mu}_1-\B{\mu}_2)+\B{\nu}\right) \\
			\mbox{s.t.}\   \quad & \B{\mu}_1,\B{\mu}_2\succeq\B{0}
		\end{split}
	\end{equation*}
     where $\B{\nu}$ denotes the vector in $\mathbb{R}^{|\set{X}||\set{Y}|}$ with component corresponding with $(x,y)$ for $x\in\set{X}$, $y\in\set{Y}$ given by $\nu(x)$, and $f^*$ is the conjugate function of $f(\hat{\V{p}})=\V{q}^T\hat{\V{p}}+I_+(\hat{\B{\up{p}}})$ that becomes
	\begin{equation*}
		f^*(\B{w})=\left\lbrace\begin{array}{cc} 0 & \text{if }\V{w}\preceq \V{q} \\ \infty & \text{otherwise}. \end{array}\right.
	\end{equation*}
	Therefore, the previous Lagrange dual becomes 
	\begin{equation*}
		\begin{split}
			\max_{\B{\mu}_1,\B{\mu}_2,\nu(x)}\quad & (\B{\tau}-\B{\lambda})^T\B{\mu}_1-(\B{\tau}+\B{\lambda})^T\B{\mu}_2+\mathbb{E}_{\up{p}_{\text{te}}(x)}\nu(x) \\
			\mbox{s.t.}\   \quad & \B{\mu}_1,\B{\mu}_2\succeq\B{0} \\
			& \B{\Phi}_{\alpha}(\B{\mu}_1-\B{\mu}_2)+\B{\nu}\preceq \B{q}
		\end{split}
	\end{equation*}
	which is equivalent to
	\begin{equation*}
		\begin{split}
			\max_{\B{\mu}_1,\B{\mu}_2}\quad & (\B{\tau}-\B{\lambda})^T\B{\mu}_1-(\B{\tau}+\B{\lambda})^T\B{\mu}_2+\mathbb{E}_{\up{p}_{\text{te}}(x)}\min_{y\in\set{Y}}\left\lbrace q(x,y)-\alpha(x)\Phi(x,y)^T(\B{\mu}_1-\B{\mu}_2) \right\rbrace \\
			\mbox{s.t.}\   \quad & \B{\mu}_1,\B{\mu}_2\succeq\B{0}.
		\end{split}
	\end{equation*}
	Taking $\B{\mu}=\B{\mu}_1-\B{\mu}_2$ the Lagrange dual problem is equivalent to
	\begin{equation*}
		\max_{\B{\mu}} \B{\tau}^T\B{\mu}+\mathbb{E}_{\up{p}_{\text{te}}(x)}\min_{y\in\set{Y}}\left\lbrace q(x,y)-\alpha(x)\Phi(x,y)^T\B{\mu}\right\rbrace -\B{\lambda}^T|\B{\mu}|
	\end{equation*}
	that has the same value as its primal $\min_{\hat{\up{p}}\in\set{U}}\mathbb{E}_{\hat{\up{p}}}q$ since the constraints defining $\set{U}$ are affine and $\set{U}\neq\emptyset$.	
	Then, we have that
	\begin{alignat}{2}
		\max_{\hat{\up{p}}\in\set{U}}\ell_{01}(\up{h},\hat{\up{p}})&=-\min_{\hat{\up{p}}\in\set{U}}\mathbb{E}_{\hat{\up{p}}}\left\lbrace \up{h}(y|x)-1\right\rbrace&&=\min_{\B{\mu}}- \B{\tau}^T\B{\mu}+\mathbb{E}_{\up{p}_{\text{te}}(x)}\max_{y\in\set{Y}}\left\lbrace1+ \alpha(x)\Phi(x,y)^T\B{\mu}-\up{h}(y|x)\right\rbrace +\B{\lambda}^T|\B{\mu}|\nonumber\\
        \max_{\hat{\up{p}}\in\set{U}}\ell_{\log}(\up{h},\hat{\up{p}})&=-\min_{\hat{\up{p}}\in\set{U}}\mathbb{E}_{\hat{\up{p}}}\log\up{h}(y|x)&&=\min_{\B{\mu}}- \B{\tau}^T\B{\mu}+\mathbb{E}_{\up{p}_{\text{te}}(x)}\max_{y\in\set{Y}}\left\lbrace \alpha(x)\Phi(x,y)^T\B{\mu}-\log\up{h}(y|x)\right\rbrace +\B{\lambda}^T|\B{\mu}|\nonumber
	\end{alignat}
 for 0-1-loss and log-loss respectively.
\end{proof}

\begin{proof}[\textbf{Proof of Theorem \ref{th_3:generalization_bound}}]
	For inequality \eqref{eq_3:th_3_bound1}, let $\set{U}_\infty$ be the uncertainty set given by the exact mean vector \mbox{$\tau_{\infty}=\mathbb{E}_{\up{p}_{\text{te}}}\Phi_{\alpha}(x,y)$}, i.e., 
    \begin{alignat}{1}\label{uncertainty_set_infty}
		\set{U}_{\infty}=\left\lbrace\right.&\up{p}\in\Delta\left(\set{X}\times\set{Y}\right):\left|\mathbb{E}_{\up{p}}\Phi_\alpha(x,y)-\B{\tau}_{\infty}\right|\preceq\B{\lambda}\notag\\
		&\left.\text{and }\up{p}(x)=\up{p}_{\text{te}}(x),\forall x\in\mathcal{X}\right\rbrace.
    \end{alignat}
    It is clear that we have $\up{p}_{\text{te}}(x,y)\in\set{U}_\infty$, then for 0-1-loss, using Lemma~\ref{lemma_1} and the definition of $\mathrm{h}(y|x)$ in \eqref{eq_3:h_equality_01}, we have that
    \begin{alignat}{1}
		R(\mathrm{h}^{\sset{U}})&\leq \overline{R}_{01}(\set{U}_\infty,\mathrm{h}^{\sset{U}})=\min_{\B{\mu}}-\B{\tau}_\infty^T\B{\mu}+\mathbb{E}_{\mathrm{p}_{\text{te}}(x)}\max_{y\in\set{Y}}\left\lbrace 1+\alpha(x)\Phi(x,y)^T\B{\mu}-\mathrm{h}(y|x)\right\rbrace \notag\\
            &\leq -\B{\tau}_\infty^T\B{\mu}^*+\mathbb{E}_{\mathrm{p}_{\text{te}}(x)}\max_{y\in\set{Y}}\left\lbrace 1+\alpha(x)\Phi(x,y)^T\B{\mu}^*-\mathrm{h}(y|x)\right\rbrace \label{ApB_43}\\
		&\leq -\B{\tau}_\infty^T\B{\mu}^*+\mathbb{E}_{\mathrm{p}_{\text{te}}(x)}\max_{y\in\set{Y}}\varphi_{01}(\B{\mu}^*,x,\alpha(x)) \label{ApB_44}\\
		&= -\B{\tau}_\infty^T\B{\mu}^*+\mathbb{E}_{\mathrm{p}_{\text{te}}(x)} \varphi_{01}(\B{\mu}^*,x,\alpha(x)) \label{ApB_45}\\
		&=R(\set{U})+(\B{\tau}-\B{\tau}_\infty)^T\B{\mu}^*-\B{\lambda}^T|\B{\mu}^*|.\label{ApB_46}
    \end{alignat}
 where, for inequality \eqref{ApB_43}-\eqref{ApB_44}, we have used the fact that $\mathrm{h}(y|x)\geq\alpha(x)\Phi(x,y)^T\B{\mu}^*+1-\varphi_{01}(\B{\mu}^*,x,\alpha(x))$ and for inequality \eqref{ApB_45}-\eqref{ApB_46} we have added and subtracted $\B{\tau}^T\B{\mu}^*$ and $\B{\lambda}^T\B{\mu}^*$, and used the definition of $R(\set{U})$ in \eqref{eq_3:R_l(U)}.
 
For log-loss, using Lemma~\ref{lemma_1} and the definition of $\mathrm{h}(y|x)$ in \eqref{eq_3:h_equality_log}, we have that
    \begin{alignat}{1}
	R(\mathrm{h}^{\sset{U}})&\leq \overline{R}_{\log}               
        (\set{U}_\infty,\mathrm{h}^{\sset{U}})=\min_{\B{\mu}}-\B{\tau}_\infty^T\B{\mu}+\mathbb{E}_{\mathrm{p}_{\text{te}}(x)}\max_{y\in\set{Y}}\left\lbrace \alpha(x)\Phi(x,y)^T\B{\mu}-\log\mathrm{h}(y|x)\right\rbrace\notag \\
	&\leq -\B{\tau}_\infty^T\B{\mu}^*+\mathbb{E}_{\mathrm{p}_{\text{te}}                  (x)}\max_{y\in\set{Y}}\left\lbrace \alpha(x)\Phi(x,y)^T\B{\mu}^*-\log\mathrm{h}       (y|x)\right\rbrace \label{ApB_47}\\
        &= -\B{\tau}_\infty^T\B{\mu}^*+\mathbb{E}_{\mathrm{p}_{\text{te}}(x)}\max_{y\in\set{Y}} \varphi_{\log}(\B{\mu}^*,x,\alpha(x)) \label{ApB_48}\\
	&= -\B{\tau}_\infty^T\B{\mu}^*+\mathbb{E}_{\mathrm{p}_{\text{te}}(x)}                 \varphi_{\log}(\B{\mu}^*,x,\alpha(x)) \label{ApB_49}\\
	&=R(\set{U})+(\B{\tau}-\B{\tau}_\infty)^T\B{\mu}^*-                 \B{\lambda}^T|\B{\mu}^*|\label{ApB_50}
    \end{alignat}
    where, for inequality \eqref{ApB_47}-\eqref{ApB_48} we have used the fact that $\log\mathrm{h}(y|x)=\alpha(x)\Phi(x,y)^T\B{\mu}^*-\varphi_{\log}(\B{\mu}^*,x,\alpha(x))$ and for inequality \eqref{ApB_49}-\eqref{ApB_50} we have added and subtracted $\B{\tau}^T\B{\mu}^*$ and $\B{\lambda}^T\B{\mu}^*$, and used the definition of $R(\set{U})$ in \eqref{eq_3:R_l(U)}.
	
 For inequality \eqref{eq_3:th_3_bound2}, note that using the definition of $\B{\mu}^*$ and \eqref{ApB_45} (resp. \eqref{ApB_49}) for 0-1-loss (resp. log-loss), we have that
	\begin{equation*}
		\begin{split}
			R(\up{h}^{\sset{U}})&\leq-\B{\tau}_\infty^T\B{\mu}^*+\mathbb{E}_{\up{p}_{\text{te}}(x)} \varphi_{\ell}(\B{\mu}^*,x,\alpha(x)) \\
			&\leq-\B{\tau}^T\B{\mu}_{\infty}+\mathbb{E}_{\up{p}_{\text{te}}(x)} \varphi_{\ell}(\B{\mu}_{\infty},x,\alpha(x)) +\B{\lambda}^T|\B{\mu}_{\infty}|+\left(\B{\tau}-\B{\tau}_{\infty}\right)^T\B{\mu}^*-\B{\lambda}^T|\B{\mu}^*| \\
			&= R^{\infty}+\B{\lambda}^T\left(|\B{\mu}_{\infty}|-|\B{\mu}^*|\right)+(\B{\tau}_{\infty}-\B{\tau})^T\B{\mu}_{\infty}+(\B{\tau}-\B{\tau}_{\infty})^T\B{\mu}^*\\
            &\leq R^{\infty}+\B{\lambda}^T\left(|\B{\mu}_{\infty}|-|\B{\mu}^*|\right)+\left|\B{\tau}-\B{\tau}_{\infty}\right|^T\left|\B{\mu}_{\infty}-\B{\mu}^*\right|.
		\end{split}	
	\end{equation*}
\end{proof}
\section{Proofs for Section \ref{Section_4}}\label{Appendix_3:Proofs_Sect4}

\begin{proof}[\textbf{Proof of Theorem \ref{th_4:bound_GKMM}}]
    The proof is analogous to Example 6.3 in \cite{Boucheron2013} that shows a Hoeffding-type inequality in Hilbert space.
    
    We consider $n+t$ independent random variables taking values in the Hilbert space $\set{H}$ as follows
	\begin{equation}
		f_i=\left\lbrace \begin{array}{ll}
			\frac{1}{n}\hat{\beta}(x_i)K(x_i) & \text{ if }i=1,2,\ldots,n\\
			&\\
			-\frac{1}{t}\hat{\alpha}(x_{i})K(x_{i}) & \text{ if }i=n+1,n+2,\ldots,n+t.
		\end{array}\right. 
	\end{equation}
and we want to bound $||\sum_{i=1}^{n+t}f_i||_{\set{H}}$.
 We have that, 
	\begin{equation}
		||f_i||_{\set{H}}\leq\left\lbrace \begin{array}{ll}
			\frac{1}{n}\frac{B}{\sqrt{D}}\kappa & \text{ if }i=1,2,\ldots,n\\
			&\\
			\frac{1}{t}\kappa & \text{ if }i=n+1,n+2,\ldots,n+t.
		\end{array}\right. 
	\end{equation}

	Taking $v=\kappa^2\left(\frac{B^2}{Dn}+\frac{1}{t} \right)$ and using the bounded differences inequality (Theorem 6.2 in \citep{Boucheron2013}), we have that, for all $l\geq\sqrt{v}$
	\begin{equation}
		\begin{split}
			\mathbb{P}\left\lbrace \left| \left| \sum_{i=1}^{n+t}f_i\right| \right|_\set{H} >l\right\rbrace &=\mathbb{P}\left\lbrace \left| \left| \sum_{i=1}^{n+t}f_i\right| \right|_\set{H} -\mathbb{E} \left| \left| \sum_{i=1}^{n+t}f_i\right| \right|_\set{H} >l-\mathbb{E} \left| \left| \sum_{i=1}^{n+t}f_i\right| \right|_\set{H} \right\rbrace\\
			&\leq \exp\left\lbrace-\frac{\left(l-\mathbb{E} \left| \left| \sum_{i=1}^{n+t}f_i\right| \right|_\set{H}  \right) ^2}{2v} \right\rbrace.
		\end{split}
	\end{equation}
	Finally, using Hölder's inequality and by independence, we have that
	\begin{equation}\nonumber
		\mathbb{E} \left| \left| \sum_{i=1}^{n+t}f_i\right| \right|_\set{H} \leq\sqrt{\mathbb{E} \left| \left| \sum_{i=1}^{n+t}f_i\right| \right|_\set{H} ^2}=\sqrt{\sum_{i=1}^{n+t}\mathbb{E}\left| \left|f_i \right| \right|_\set{H} ^2}\leq\sqrt{v}.\nonumber
	\end{equation}
	Therefore, 
    \begin{equation}\nonumber
			\exp\left\lbrace-\frac{\left(l-\sqrt{v}  \right) ^2}{2v} \right\rbrace =\exp\left\lbrace-\frac{\left(l-\sqrt{\kappa^2\left(\frac{B^2}{Dn}+\frac{1}{t} \right)}  \right) ^2}{2\kappa^2\left(\frac{B^2}{Dn}+\frac{1}{t} \right)} \right\rbrace
    \end{equation}
    so that, 
	\begin{equation}\nonumber
		\left| \left| \frac{1}{n}\sum_{i=1}^n\hat{\beta}(x_i)K(x_i)-\frac{1}{t}\sum_{i=n+1}^{n+t}\hat{\alpha}(x_{n+i})K(x_{n+i})\right| \right|_\set{H} \leq\left(1+\sqrt{2\log\frac{1}{\delta}}\right) \kappa\sqrt{\left(\frac{B^2}{Dn}+\frac{1}{t} \right) }
	\end{equation}
with probability at least $1-\delta$.
\end{proof}
\section{Quadratic version of \ac{DW-KMM}}\label{Appendix_4:Quadratic_DKMM}

The convex optimization in \eqref{eq_4:DW-KMMempirical} is a quadratic problem since the squared norm in $\set{H}$ can be written as
\vspace{-0.2cm}
\begin{alignat}{1}
		\Bigg\| \frac{1}{t}\sum_{i=1}^{t}\alpha^{(i)}K(x_{n+i})-&\frac{1}{n}\sum_{i=1}^n\beta^{(i)}K(x_i)\Bigg\| ^2_\set{H} \notag\\
		=&\frac{1}{t^2}\sum_{i,j=1}^{t}\alpha^{(i)}\alpha^{(j)}k(x_{n+i},x_{n+j})+\frac{1}{n^2}\sum_{i,j=1}^{n}\beta^{(i)}\beta^{(j)}k(x_i,x_j)  -\frac{2}{nt}\sum_{i=1}^{t}\sum_{j=1}^n\alpha^{(i)}\beta^{(j)}k(x_{n+i},x_j) \notag\\
		=&\frac{\B{\alpha}^T}{t}\begin{bmatrix}k(x_{n+1},x_{n+1}) & \cdots & k(x_{n+1},x_{n+t})  \\ \vdots & \ddots & \vdots \\ k(x_{n+t},x_{n+1}) & \cdots & k(x_{n+t},x_{n+t}) \end{bmatrix}\frac{\B{\alpha}}{t}+\frac{\B{\beta}^T}{n}\begin{bmatrix}k(x_1,x_1) & \cdots & k(x_1,x_n)  \\ \vdots & \ddots & \vdots \\ k(x_n,x_1) & \cdots & k(x_n,x_n) \end{bmatrix}\frac{\B{\beta}}{n} \notag\\
		&-2\frac{\B{\beta}^T}{n}\begin{bmatrix}k(x_1,x_{n+1}) & \cdots & k(x_1,x_{n+t})  \\ \vdots & \ddots & \vdots \\ k(x_n,x_{n+1}) & \cdots & k(x_{n},x_{n+t}) \end{bmatrix}\frac{\B{\alpha}}{t}\notag\\
	=&\left[\B{\beta}^T/n,-\B{\alpha}^T/t\right]\V{K}\begin{bmatrix}\B{\beta}/n\\-\B{\alpha}/t\end{bmatrix}\notag
\end{alignat}
where $\V{K}$ is the kernel matrix given by $\up{K}^{(i,j)}=k(x_i,x_j)$.

 Therefore, the optimization problem \eqref{eq_4:DW-KMMempirical} is equivalent to the quadratic optimization problem
 \begin{alignat}{2}\label{eq_5:GKMMquadratic}
 			\min_{\B{\alpha},\B{\beta}}\  \  & &&\Big[\B{\beta}^T/n,-\B{\alpha}^T/t \Big]\V{K}\begin{bmatrix}\B{\beta}/n\\
 				-\B{\alpha}/t\end{bmatrix}\notag\\
 			\text{s.t. }\  \  & &&\textbf{0}\preceq\B{\beta}\preceq (B/\sqrt{D})\V{1},\quad \textbf{0}\preceq \B{\alpha}\preceq\textbf{1}\notag\\
             & &&\left|\B{\beta}^T\B{1}/n-\B{\alpha}^T\B{1}/t\right|\leq \epsilon\\
 			& &&\left| \left|\B{\alpha}-\mathbf{1} \right| \right|\leq  \left(1-\frac{1}{\sqrt{D}} \right)\sqrt{t}\notag.
 \end{alignat}

\section{Implementation
details and additional experimental results}\label{Appendix_5:Experiments}

This appendix details the datasets and settings used for the experiments in Section~\ref{Section_6:Experiments} and shows additional experiments.

For the experiments in Section~\ref{Section_6:Experiments}, we have considered four binary classification datasets, available in the UCI repository \cite{Dua2019}, and previously used in multiple papers on covariate shift adaptation \cite{Gretton2009,Huang2006,Kanamori2009,Mazaheri2020,Wen2014}. 
In addition, we use the dataset “News20groups” that is intrinsically affected by covariate shift \cite{Zhang2013}.

Table~\ref{Table:Experiment_Data} details the characteristics of the datasets used in the experiments. 
The table also shows the parameter $\sigma$ used in the computation of the kernel matrix $\V{K}$ for the RuLSIF, \ac{KMM} and \ac{DW-KMM} methods, which is determined using the common heuristic based on nearest neighbors with $K=50$, as is done in \cite{Wen2014}. 
For the results obtained using the flattening method in \cite{Shimodaira2000} and the RuLSIF method in \cite{Yamada2011} we considered the hyperparameter $\gamma=0.5$, which is the default value used in those papers.

\begin{table*}[ht]
\small
\centering
\caption{Datasets used in the experiments.}
\vspace{0.1cm}
\begin{tabular}{lccccc}
\toprule
\multirow{2}{*}{Dataset} & \multirow{2}{*}{Covariates} & \multicolumn{2}{c}{\multirow{2}{*}{Samples}} & Ratio of & \multirow{2}{*}{$\sigma$} \\ 
&&&&majority class&\\ 
\midrule
\multirow{2}{*}{Blood} & \multirow{2}{*}{3} & \multicolumn{2}{c}{\multirow{2}{*}{748}} & \multirow{2}{*}{76.20$\%$} & \multirow{2}{*}{0.7491}\\ 
&&&&&\\
\hline
\multirow{2}{*}{BreastCancer} & \multirow{2}{*}{9} & \multicolumn{2}{c}{\multirow{2}{*}{683}} & \multirow{2}{*}{65.01$\%$} & \multirow{2}{*}{1.6064}\\ 
&&&&&\\
\hline
\multirow{2}{*}{Haberman} & \multirow{2}{*}{3} & \multicolumn{2}{c}{\multirow{2}{*}{306}} & \multirow{2}{*}{75.53$\%$} & \multirow{2}{*}{1.3024}\\  
&&&&&\\  
\hline
\multirow{2}{*}{Ringnorm} & \multirow{2}{*}{20} & \multicolumn{2}{c}{\multirow{2}{*}{7400}} & \multirow{2}{*}{50.49$\%$} & \multirow{2}{*}{3.8299}\\ 
&&&&&\\
\hline
\multirow{2}{*}{comp vs sci} & \multirow{2}{*}{1000}  & \multirow{2}{*}{5309} & \multirow{2}{*}{3534} & \multirow{2}{*}{55.31$\%$} & \multirow{2}{*}{23.5628}\\ 
&&&&&\\
\hline
\multirow{2}{*}{comp vs talk} & \multirow{2}{*}{1000}  & \multirow{2}{*}{4888} & \multirow{2}{*}{3256} & \multirow{2}{*}{60.06$\%$} & \multirow{2}{*}{23.4890}\\ 
&&&&&\\
\hline
\multirow{2}{*}{rec vs sci} & \multirow{2}{*}{1000}  & \multirow{2}{*}{4762} & \multirow{2}{*}{3169} & \multirow{2}{*}{50.17$\%$} & \multirow{2}{*}{24.5642}\\ 
&&&&&\\
\hline
\multirow{2}{*}{rec vs talk} & \multirow{2}{*}{1000}  & \multirow{2}{*}{4341} & \multirow{2}{*}{2891} & \multirow{2}{*}{55.02$\%$} & \multirow{2}{*}{25.1129}\\ 
&&&&&\\
\hline
\multirow{2}{*}{sci vs talk} & \multirow{2}{*}{1000}  & \multirow{2}{*}{4325} & \multirow{2}{*}{2880} & \multirow{2}{*}{54.85$\%$} & \multirow{2}{*}{24.8320}\\ 
&&&&&\\
\bottomrule
\end{tabular}
\label{Table:Experiment_Data}
\end{table*}

In the additional experiments we study the effectiveness of the proposed selection method for hyperparameter $D$.
Specifically, Tables~\ref{Table:D_01} and \ref{Table:D_log} show the average classification error varying the value of $D$ for the datasets and covariate shifts shown in Table~\ref{Table1:Experiment_1}.
The first column of these tables shows the classification error obtained when selecting $D$ with the proposed method that minimizes the minimax risk, while the other columns show the classification error obtained using specific values of $D$.
The values of the hyperparameter $D$ have been chosen based on the last inequality in the optimization problem \eqref{eq_4:DW-KMMempirical}. 
Specifically, we take the values for $D$ such that $1-1/\sqrt{D}\in\left\lbrace0,0.1,\ldots,0.9\right\rbrace$. 
As can be seen from the tables, the proposed selection method results in performances near those obtained with the best values of $D$.

\begin{table*}[ht]
\small
\centering
\caption{Classification error in 21 scenarios using DW-GCS methods with 0-1-loss varying the value of the hyperparameter $D$.}
\vspace{0.1cm}
\begin{tabular}{p{2.03cm}>{\centering\arraybackslash}p{1cm}>{\centering\arraybackslash}p{0.89cm}>{\centering\arraybackslash}>{\centering\arraybackslash}p{0.89cm}>{\centering\arraybackslash}p{0.89cm}>{\centering\arraybackslash}p{0.89cm}>{\centering\arraybackslash}p{0.89cm}>{\centering\arraybackslash}p{0.89cm}>{\centering\arraybackslash}p{0.89cm}>{\centering\arraybackslash}p{0.89cm}>{\centering\arraybackslash}p{0.89cm}>{\centering\arraybackslash}p{1cm}}
\toprule
\multirow{2}{*}{Dataset} & proposed & \multirow{2}{*}{$D=1$} & \multirow{2}{*}{$D=1.2$} & \multirow{2}{*}{$D=1.6$} & \multirow{2}{*}{$D=2$} & \multirow{2}{*}{$D=2.8$} & \multirow{2}{*}{$D=4$} & \multirow{2}{*}{$D=6.3$} & \multirow{2}{*}{$D=11$} & \multirow{2}{*}{$D=25$} & \multirow{2}{*}{$D=100$}\\ 
&selection&&&&&&&\\
\midrule
\textbf{Blood}&&&&&&&&\\
Feature 1 & $0.30$ & $0.32$ & $0.31$ & $0.31$ & $0.31$ & $0.30$ & $0.30$ & $0.30$ & $\V{0.29}$ & $\V{0.29}$ & $0.30$ \\ 
Feature 2 &$0.38$ & $0.40$ & $0.40$ & $0.40$ & $0.40$ & $0.39$ & $0.39$ & $0.38$ & $0.38$ & $\V{0.37}$ & $0.38$ \\ 
Feature 3 & $0.34$ & $0.39$ & $0.37$ & $0.36$ & $0.36$ & $0.35$ & $0.34$ & $0.34$ & $0.34$ & $\V{0.33}$ & $0.34$ \\ 
PCA  & $0.28$ & $0.32$ & $0.30$ & $0.29$ & $0.28$ & $\V{0.27}$ & $\V{0.27}$ & $0.28$ & $0.28$ & $0.28$ & $0.28$ \\ 
\hline 
\textbf{BreastCancer}&&&&&&&&\\
Feature 1 & $\V{0.04}$ & $0.05$ & $0.05$ & $\V{0.04}$ & $\V{0.04}$ & $\V{0.04}$ & $\V{0.04}$ & $0.05$ & $0.05$ & $\V{0.04}$ & $\V{0.04}$ \\ 
Feature 2 & $\V{0.04}$ & $0.06$ & $0.06$ & $0.05$ & $0.05$ & $0.05$ & $0.06$ & $0.06$ & $0.06$ & $\V{0.04}$ & $\V{0.04}$ \\ 
Feature 3 & $\V{0.04}$ & $0.06$ & $0.05$ & $0.05$ & $0.05$ & $0.06$ & $0.05$ & $0.06$ & $0.05$ & $\V{0.04}$ & $\V{0.04}$ \\ 
PCA  & $\V{0.02}$ & $0.03$ & $0.03$ & $\V{0.02}$ & $\V{0.02}$ & $0.03$ & $0.03$ & $0.03$ & $\V{0.02}$ & $\V{0.02}$ & $\V{0.02}$ \\ 
\hline 
\textbf{Haberman}&&&&&&&&\\
Feature 1 & $\V{0.28}$ & $0.39$ & $0.36$ & $0.33$ & $0.30$ & $0.29$ & $\V{0.28}$ & $\V{0.28}$ & $\V{0.28}$ & $\V{0.28}$ & $\V{0.28}$ \\ 
Feature 2 & $\V{0.29}$ & $0.39$ & $0.37$ & $0.35$ & $0.33$ & $0.32$ & $0.30$ & $0.30$ & $0.30$ & $0.30$ & $\V{0.29}$ \\ 
Feature 3 & $0.35$ & $0.46$ & $0.45$ & $0.43$ & $0.40$ & $0.37$ & $0.35$ & $0.35$ & $\V{0.34}$ & $\V{0.34}$ & $\V{0.34}$ \\ 
PCA  & $0.30$ & $0.40$ & $0.38$ & $0.35$ & $0.32$ & $0.30$ & $\V{0.29}$ & $\V{0.29}$ & $\V{0.29}$ & $0.30$ & $\V{0.29}$ \\ 
\hline 
\textbf{Ringnorm}&&&&&&&&\\
Feature 1 & $\V{0.25}$ & $\V{0.25}$ & $\V{0.25}$ & $\V{0.25}$ & $\V{0.25}$ & $\V{0.25}$ & $\V{0.25}$ & $\V{0.25}$ & $\V{0.25}$ & $\V{0.25}$ & $\V{0.25}$ \\  
Feature 2 & $\V{0.25}$ & $\V{0.25}$ & $0.26$ & $\V{0.25}$ & $\V{0.25}$ & $\V{0.25}$ & $\V{0.25}$ & $\V{0.25}$ & $\V{0.25}$ & $\V{0.25}$ & $\V{0.25}$ \\ 
Feature 3 & $\V{0.25}$ & $0.26$ & $0.26$ & $0.26$ & $0.26$ & $0.26$ & $0.26$ & $0.26$ & $0.26$ & $\V{0.25}$ & $\V{0.25}$ \\ 
PCA  & $0.27$ & $0.30$ & $0.29$ & $0.29$ & $0.28$ & $0.28$ & $0.27$ & $0.27$ & $\V{0.26}$ & $0.27$ & $0.27$ \\ 
\hline
\textbf{20 Newsgroups}&&&&&&&&\\
comp vs sci & $0.22$ & $0.25$ & $0.24$ & $0.23$ & $0.22$ & $\V{0.21}$ & $\V{0.21}$ & $\V{0.21}$ & $\V{0.21}$ & $\V{0.21}$ & $\V{0.21}$ \\ 
comp vs talk & $0.11$ & $0.17$ & $0.16$ & $0.14$ & $0.12$ & $0.11$ & $\V{0.10}$ & $\V{0.10}$ & $\V{0.10}$ & $0.11$ & $0.11$ \\ 
rec vs sci & $0.17$ & $0.19$ & $0.18$ & $0.18$ & $0.17$ & $0.17$ & $\V{0.16}$ & $\V{0.16}$ & $\V{0.16}$ & $\V{0.16}$ & $\V{0.16}$ \\ 
rec vs talk & $0.15$ & $0.18$ & $0.17$ & $0.16$ & $0.15$ & $0.15$ & $\V{0.14}$ & $\V{0.14}$ & $\V{0.14}$ & $\V{0.14}$ & $\V{0.14}$ \\   
sci vs talk & $0.20$ & $0.22$ & $0.21$ & $0.20$ & $0.19$ & $0.19$ & $\V{0.18}$ & $\V{0.18}$ & $\V{0.18}$ & $0.19$ & $0.19$ \\ 
\bottomrule
\end{tabular}
\label{Table:D_01}
\end{table*}

\begin{table*}[ht]
\small
\centering
\caption{Classification error in 21 scenarios using DW-GCS methods with log-loss varying the value of the hyperparameter $D$.}
\vspace{0.1cm}
\begin{tabular}{p{2.05cm}>{\centering\arraybackslash}p{1cm}>{\centering\arraybackslash}p{0.89cm}>{\centering\arraybackslash}>{\centering\arraybackslash}p{0.89cm}>{\centering\arraybackslash}p{0.89cm}>{\centering\arraybackslash}p{0.89cm}>{\centering\arraybackslash}p{0.89cm}>{\centering\arraybackslash}p{0.89cm}>{\centering\arraybackslash}p{0.89cm}>{\centering\arraybackslash}p{0.89cm}>{\centering\arraybackslash}p{0.89cm}>{\centering\arraybackslash}p{1cm}}
\toprule
\multirow{2}{*}{Dataset} & proposed & \multirow{2}{*}{$D=1$} & \multirow{2}{*}{$D=1.2$} & \multirow{2}{*}{$D=1.6$} & \multirow{2}{*}{$D=2$} & \multirow{2}{*}{$D=2.8$} & \multirow{2}{*}{$D=4$} & \multirow{2}{*}{$D=6.3$} & \multirow{2}{*}{$D=11$} & \multirow{2}{*}{$D=25$} & \multirow{2}{*}{$D=100$}\\ 
&selection&&&&&&&\\
\midrule
\textbf{Blood}&&&&&&&&\\
Feature 1 & $0.31$ & $0.32$ & $0.32$ & $0.32$ & $0.31$ & $0.30$ & $0.30$ & $0.30$ & $\V{0.29}$ & $\V{0.29}$ & $0.30$ \\ 
Feature 2 & $\V{0.38}$ & $0.41$ & $0.40$ & $0.40$ & $0.40$ & $0.39$ & $\V{0.38}$ & $\V{0.38}$ & $\V{0.38}$ & $\V{0.38}$ & $\V{0.38}$ \\ 
Feature 3 & $0.35$ & $0.38$ & $0.38$ & $0.37$ & $0.36$ & $0.36$ & $0.35$ & $\V{0.34}$ & $\V{0.34}$ & $\V{0.34}$ & $\V{0.34}$ \\ 
PCA  & $\V{0.28}$ & $0.32$ & $0.30$ & $0.29$ & $0.29$ & $\V{0.28}$ & $\V{0.28}$ & $\V{0.28}$ & $\V{0.28}$ & $\V{0.28}$ & $\V{0.28}$ \\ 
\hline 
\textbf{BreastCancer}&&&&&&&&\\
Feature 1 & $\V{0.04}$ & $0.05$ & $0.05$ & $0.05$ & $0.05$ & $\V{0.04}$ & $\V{0.04}$ & $0.05$ & $0.05$ & $\V{0.04}$ & $\V{0.04}$ \\ 
Feature 2 & $\V{0.04}$ & $0.06$ & $0.06$ & $0.06$ & $0.06$ & $0.06$ & $0.06$ & $0.06$ & $0.06$ & $\V{0.04}$ & $\V{0.04}$ \\ 
Feature 3 & $\V{0.04}$ & $0.06$ & $0.06$ & $0.06$ & $0.06$ & $0.06$ & $0.05$ & $0.06$ & $0.05$ & $\V{0.04}$ & $\V{0.04}$ \\ 
PCA  & $\V{0.02}$ & $0.03$ & $0.03$ & $0.03$ & $0.03$ & $0.03$ & $0.03$ & $0.03$ & $\V{0.02}$ & $\V{0.02}$ & $\V{0.02}$ \\ 
\hline 
\textbf{Haberman}&&&&&&&&\\
Feature 1 & $0.29$ & $0.38$ & $0.36$ & $0.34$ & $0.31$ & $0.30$ & $0.29$ & $0.29$ & $0.29$ & $0.29$ & $\V{0.28}$ \\ 
Feature 2 & $\V{0.30}$ & $0.39$ & $0.37$ & $0.35$ & $0.33$ & $0.32$ & $0.31$ & $0.31$ & $0.31$ & $0.31$ & $\V{0.30}$ \\ 
Feature 3 & $0.36$ & $0.46$ & $0.44$ & $0.43$ & $0.40$ & $0.37$ & $0.36$ & $0.35$ & $0.35$ & $0.35$ & $\V{0.34}$ \\ 
PCA & $0.31$ & $0.39$ & $0.38$ & $0.36$ & $0.33$ & $0.31$ & $\V{0.30}$ & $\V{0.30}$ & $\V{0.30}$ & $0.31$ & $\V{0.30}$ \\ 
\hline 
\textbf{Ringnorm}&&&&&&&&\\
Feature 1 & $\V{0.25}$ & $\V{0.25}$ & $\V{0.25}$ & $\V{0.25}$ & $\V{0.25}$ & $\V{0.25}$ & $\V{0.25}$ & $\V{0.25}$ & $\V{0.25}$ & $\V{0.25}$ & $\V{0.25}$ \\  
Feature 2 & $\V{0.25}$ & $\V{0.25}$ & $\V{0.25}$ & $\V{0.25}$ & $\V{0.25}$ & $\V{0.25}$ & $\V{0.25}$ & $\V{0.25}$ & $\V{0.25}$ & $0.24$ & $\V{0.25}$ \\
Feature 3 & $\V{0.25}$ & $0.26$ & $0.26$ & $0.26$ & $\V{0.25}$ & $0.26$ & $\V{0.25}$ & $\V{0.25}$ & $\V{0.25}$ & $\V{0.25}$ & $\V{0.25}$ \\ 
PCA  & $\V{0.26}$ & $0.30$ & $0.29$ & $0.29$ & $0.28$ & $0.27$ & $0.27$ & $\V{0.26}$ & $\V{0.26}$ & $\V{0.26}$ & $0.27$ \\ 
\hline
\textbf{20 Newsgroups}&&&&&&&&\\
comp vs sci & $0.22$ & $0.25$ & $0.24$ & $0.23$ & $0.22$ & $\V{0.21}$ & $\V{0.21}$ & $\V{0.21}$ & $\V{0.21}$ & $\V{0.21}$ & $\V{0.21}$ \\ 
comp vs talk & $0.11$ & $0.17$ & $0.16$ & $0.14$ & $0.12$ & $0.11$ & $\V{0.10}$ & $\V{0.10}$ & $\V{0.10}$ & $0.11$ & $0.11$ \\ 
rec vs sci & $0.17$ & $0.19$ & $0.18$ & $0.18$ & $0.17$ & $0.17$ & $\V{0.16}$ & $\V{0.16}$ & $\V{0.16}$ & $\V{0.16}$ & $\V{0.16}$ \\ 
rec vs talk & $0.15$ & $0.18$ & $0.17$ & $0.16$ & $0.15$ & $0.15$ & $\V{0.14}$ & $\V{0.14}$ & $\V{0.14}$ & $\V{0.14}$ & $\V{0.14}$ \\   
sci vs talk & $0.20$ & $0.22$ & $0.21$ & $0.20$ & $0.19$ & $0.19$ & $\V{0.18}$ & $\V{0.18}$ & $\V{0.18}$ & $0.19$ & $0.19$ \\ 
\bottomrule
\end{tabular}
\label{Table:D_log}
\end{table*}

\end{document}